\documentclass[10pt,twocolumn,letterpaper]{article}

\usepackage{cvpr}
\usepackage{times}
\usepackage{epsfig}
\usepackage[pagebackref=true,breaklinks=true,letterpaper=true,colorlinks,bookmarks=false]{hyperref}

\usepackage[utf8]{inputenc} % allow utf-8 input
\usepackage[T1]{fontenc}    % use 8-bit T1 fonts
\usepackage{url}            % simple URL typesetting
\usepackage{booktabs}       % professional-quality tables
\usepackage{amsfonts}       % blackboard math symbols
\usepackage{nicefrac}       % compact symbols for 1/2, etc.
\usepackage{microtype}      % microtypography

\usepackage{graphicx}
\usepackage{subfigure}
\usepackage{booktabs} % for professional tables
\usepackage{bm}
\usepackage{amsmath, amssymb, amsthm}
\usepackage{algorithm,algorithmic}
\usepackage{mathtools}
\usepackage{adjustbox}
\usepackage{enumitem}
\usepackage{mdwlist}
\usepackage{physics}
\usepackage{framed}
\usepackage{dsfont}
\usepackage{tikz}
\usepackage{tkz-graph}
\usepackage{pgfplots}
\usepackage{multirow}
\usepackage{mdwtab}
\usepackage{caption}
\usepackage{wrapfig}
\usetikzlibrary{fit,positioning,shapes,backgrounds,arrows.meta,positioning,calc}

\usepackage[capitalize]{cleveref}
\newtheorem{thm}{Theorem}
\newtheorem{lemma}{Lemma}

\newcommand{\minimize}{\operatorname*{minimize}}
\newcommand{\maximize}{\operatorname*{maximize}}
\newcommand{\bfx}{\mathbf{x}}

\newcommand{\bfz}{\mathbf{z}}
\newcommand{\bfh}{\mathbf{h}}
\newcommand{\bfc}{\mathbf{c}}

\newcommand{\bftheta}{\boldsymbol\theta}
\newcommand{\reals}{\mathbb R}
\newcommand{\argmin}{\mathop{\rm argmin}}

\newcommand{\algorithmicinput}{\textbf{input}}
\newcommand{\algorithmicoutput}{\textbf{output}}
\newcommand{\INPUT}{\item[\algorithmicinput]}
\newcommand{\OUTPUT}{\item[\algorithmicoutput]}

\def\cvprPaperID{6535} % *** Enter the CVPR Paper ID here

\title{End-to-End Efficient Representation Learning\\
via Cascading Combinatorial Optimization}

\author{
Yeonwoo Jeong, Yoonsung Kim, Hyun Oh Song\\
Department of Computer Science and Engineering, Seoul National University, Seoul, Korea\\
{\tt\small \{yeonwoo, yskim227, hyunoh\}@mllab.snu.ac.kr}
}

\begin{document}

\maketitle

\begin{abstract}
We develop hierarchically quantized efficient embedding representations for similarity-based search and show that this representation provides not only the state of the art performance on the search accuracy but also provides several orders of speed up during inference. The idea is to hierarchically quantize the representation so that the quantization granularity is greatly increased while maintaining the accuracy and keeping the computational complexity low. We also show that the problem of finding the optimal sparse compound hash code respecting the hierarchical structure can be optimized in polynomial time via minimum cost flow in an equivalent flow network. This allows us to train the method end-to-end in a mini-batch stochastic gradient descent setting. Our experiments on Cifar100 and ImageNet datasets show the state of the art search accuracy while providing several orders of magnitude search speedup respectively over exhaustive linear search over the dataset.
\end{abstract}

\vspace{-1.5em}
%%%%%%%%% BODY TEXT
\section{Introduction}
\vspace{-0.5em}
Learning the feature embedding representation that preserves the notion of similarities among the data is of great practical importance in machine learning and vision and is at the basis of modern similarity-based search \cite{facenet, npairs}, verification \cite{deepface}, clustering \cite{seanbell}, retrieval \cite{liftedstruct,facility}, zero-shot learning \cite{zeroshot1,zeroshot2}, and other related tasks. In this regard, deep metric learning methods \cite{seanbell,facenet,npairs} have shown advances in various embedding tasks by training deep convolutional neural networks end-to-end encouraging similar pairs of data to be close to each other and dissimilar pairs to be farther apart in the embedding space. 

Despite the progress in improving the embedding representation accuracy, improving the inference efficiency and scalability of the representation in an end-to-end optimization framework is relatively less studied. Practitioners deploying the method on large-scale applications often resort to employing post-processing techniques such as embedding thresholding \cite{agrawal2014,zhai2017} and vector quantization \cite{survey_learningtohash} at the cost of the loss in the representation accuracy. Recently, Jeong \& Song \cite{jeong2018} proposed an end-to-end learning algorithm for quantizable representations which jointly optimizes the quality of the convolutional neural network based embedding representation and the performance of the corresponding sparsity constrained compound binary hash code and showed significant retrieval speedup on ImageNet \cite{imagenet} without compromising the accuracy.

%The inference speedup with deep sparse compound hash codes such as the setting in \cite{jeong2018} is approximately equal to the ratio between the the number of representable hash buckets ($d$) and the squared sparsity ($k_s^2$) of the compound hash code where the number of representable hash buckets correspond to the number of activations at the last layer of the hashing neural network. 

In this work, we seek to learn hierarchically quantizable representations and propose a novel end-to-end learning method significantly increasing the quantization granularity while keeping the time and space complexity manageable so the method can still be efficiently trained in a mini-batch stochastic gradient descent setting. Besides the efficiency issues, however, naively increasing the quantization granularity could cause a severe degradation in the search accuracy or lead to dead buckets hindering the search speedup. 

To this end, our method jointly optimizes both the sparse compound hash code and the corresponding embedding representation respecting a hierarchical structure. We alternate between performing cascading optimization of the optimal sparse compound hash code per each level in the hierarchy and updating the neural network to adjust the corresponding embedding representations at the active bits of the compound hash code.

Our proposed learning method outperforms both the reported results in \cite{jeong2018} and the state of the art deep metric learning methods \cite{facenet, npairs} in retrieval and clustering tasks on Cifar-100 \cite{cifar100} and ImageNet \cite{imagenet} datasets while, to the best of our knowledge, providing the highest reported inference speedup on each dataset over exhaustive linear search.
\vspace{-0.8em}
\section{Related works}\vspace{-0.5em}
Embedding representation learning with neural networks has its roots in Siamese networks \cite{signatureVerification, contrastive} where it was trained end-to-end to pull similar examples close to each other and push dissimilar examples at least some margin away from each other in the embedding space. \cite{signatureVerification} demonstrated the idea could be used for signature verification tasks. The line of work since then has been explored in wide variety of practical applications such as face recognition \cite{deepface}, domain adaptation \cite{domaintransduction}, zero-shot learning \cite{zeroshot1,zeroshot2}, video representation learning \cite{triplet_video}, and similarity-based interior design \cite{seanbell}, etc.

Another line of research focuses on learning binary hamming ranking \cite{xia2014supervised, zhao2015deep, hammingmetric, li2017deep} representations via neural networks. Although comparing binary hamming codes is more efficient than comparing continuous embedding representations, this still requires the linear search over the entire dataset which is not likely to be as efficient for large scale problems. \cite{cao2016deep, liu2017learning} seek to vector quantize the dataset and back propagate the metric loss, however, it requires repeatedly running k-means clustering on the entire dataset during training with prohibitive computational complexity.

We seek to jointly learn the hierarchically quantizable embedding representation and the corresponding sparsity constrained binary hash code in an efficient mini-batch based end-to-end learning framework. Jeong \& Song \cite{jeong2018} motivated maintaining the hard constraint on the sparsity of hash code to provide guaranteed retrieval inference speedup by only considering $k_s$ out of $d$ buckets and thus avoiding linear search over the dataset. We also explicitly maintain this constraint, but at the same time, greatly increasing the number of representable buckets by imposing an efficient hierarchical structure on the hash code to unlock significant improvement in the speedup factor.
\vspace{-0.5em}

\section{Problem formulation}\vspace{-0.5em}
Consider the following hash function 
\[r(\bfx)=\argmin_{\bfh \in \{0,1\}^d} -f(\bfx; \bftheta)^\intercal \bfh~\] under the constraint that $\|\bfh\|_1 = k_s$. The idea is to optimize the weights in the neural network $f(\cdot; \bftheta): \mathcal{X} \rightarrow \reals^d$, take $k_s$ highest activation dimensions, activate the corresponding dimensions in the binary compound hash code $\bfh$, and hash the data $\bfx \in \mathcal{X}$ into the corresponding active buckets of a hash table $\mathcal{H}$. During inference, a query $\bfx_q$ is given, and all the hashed items in the $k_s$ active bits set by the hash function $r(\bfx_q)$ are retrieved as the candidate nearest items. Often times \cite{survey_learningtohash}, these candidates are reranked based on the euclidean distance in the base embedding representation $f(\cdot; \bftheta)$ space. 

Given a query $\bfh_q$, the expected number of retrieved items is $\sum_{i\neq q} \Pr(\bfh_i^\intercal \bfh_q \neq 0)$. Then, the expected speedup factor \cite{jeong2018} (SUF) is the ratio between the total number of items and the expected number of retrieved items. Concretely, it becomes $(\Pr(\bfh_i^\intercal \bfh_q \neq 0))^{-1} = (1 - {d-k_s \choose k_s}/{d \choose k_s})^{-1}$. In case $d \gg k_s$, this ratio approaches $d/{k_s}^2$.

Now, suppose we design a hash function $r(\bfx)$ so that the function has total $\mathrm{dim}(r(\bfx))=d^k$ (\ie~ exponential in some integer parameter $k>1$) indexable buckets. The expected speedup factor \cite{jeong2018} approaches $d^k/k_s^2$ which means the query time speedup increases linearly with the number of buckets. However, naively increasing the bucket size for higher speedup has several major downsides. First, the hashing network has to output and hold $d^k$ activations in the memory at the final layer which can be unpractical in terms of the space efficiency for large scale applications. Also, this could also lead to \emph{dead buckets} which are under-utilized and degrade the search speedup. On the other hand, hashing the items uniformly at random among the buckets could help to alleviate the dead buckets but this could lead to a severe drop in the search accuracy.

Our approach to this problem of maintaining a large number of representable buckets while preserving the accuracy and keeping the computational complexity manageable is to enforce a hierarchy among the optimal hash codes in an efficient tree structure. First, we use $\mathrm{dim}(f(\bfx)) = d k$ number of activations instead of $d^k$ activations in the last layer of the hash network. Then, we define the unique mapping between the $dk$ activations to $d^k$ buckets by the following procedure. 

Denote the hash code as $\widetilde\bfh = [\bfh^1, \ldots, \bfh^k] \in \{0,1\}^{d\times k}$ where $\| \bfh^v \|_1 = 1 ~~\forall v \neq k$ and $\| \bfh^k \|_1 = k_s$. The superscript denotes the level index in the hierarchy. Now, suppose we construct a tree $\mathcal{T}$ with branching factor $d$, depth $k$ where the root node has the level index of $0$. Let each $d^k$ leaf node in $\mathcal{T}$ represent a bucket indexed by the hash function $r(\bfx)$. Then, we can interpret each $\bfh^v$ vector to indicate the branching from depth $v-1$ to depth $v$ in $\mathcal{T}$. Note, from the construction of $\widetilde\bfh$, the branching is unique until level $k-1$, but the last branching to the leaf nodes is multi-way because $k_s$ bits are set due to the sparsity constraint at level $k$. \Cref{fig:hash_tree} illustrates an example translation from the given hash activation to the tree bucket index for $k\!=\!2$ and $k_s\!=\!2$. Concretely, the hash function $r(\bfx): \mathbb{R}^{d\times k} \rightarrow \{0,1\}^{d^k}$ can be expressed compactly as \Cref{eqn:hash_function}. \vspace{-0.5em}
\small
\begin{align}
\label{eqn:hash_function}
    r&(\bfx) = \bigotimes_{v=1}^k~ \argmin_{\substack{\bfh^v}} -\left(f(\bfx; \bftheta)^v\right)^\intercal \bfh^v \\
    &\text{subject to } \|\bfh^v\|_1=\begin{dcases}1 & \forall v\neq k\\ k_s &v = k \end{dcases} \text{ and } \bfh^v \in \{0,1\}^d \nonumber
\end{align}
\normalsize
where $\bigotimes$ denotes the tensor multiplication operator between two vectors. The following section discusses how to find the optimal hash code $\widetilde\bfh$ and the corresponding activation $f(\bfx; \bftheta) = [f(\bfx; \bftheta)^1, \ldots, f(\bfx; \bftheta)^k] \in \reals^{d \times k}$ respecting the hierarchical structure of the code.

\begin{figure}[ht]
\centering
\includegraphics[width=\columnwidth]{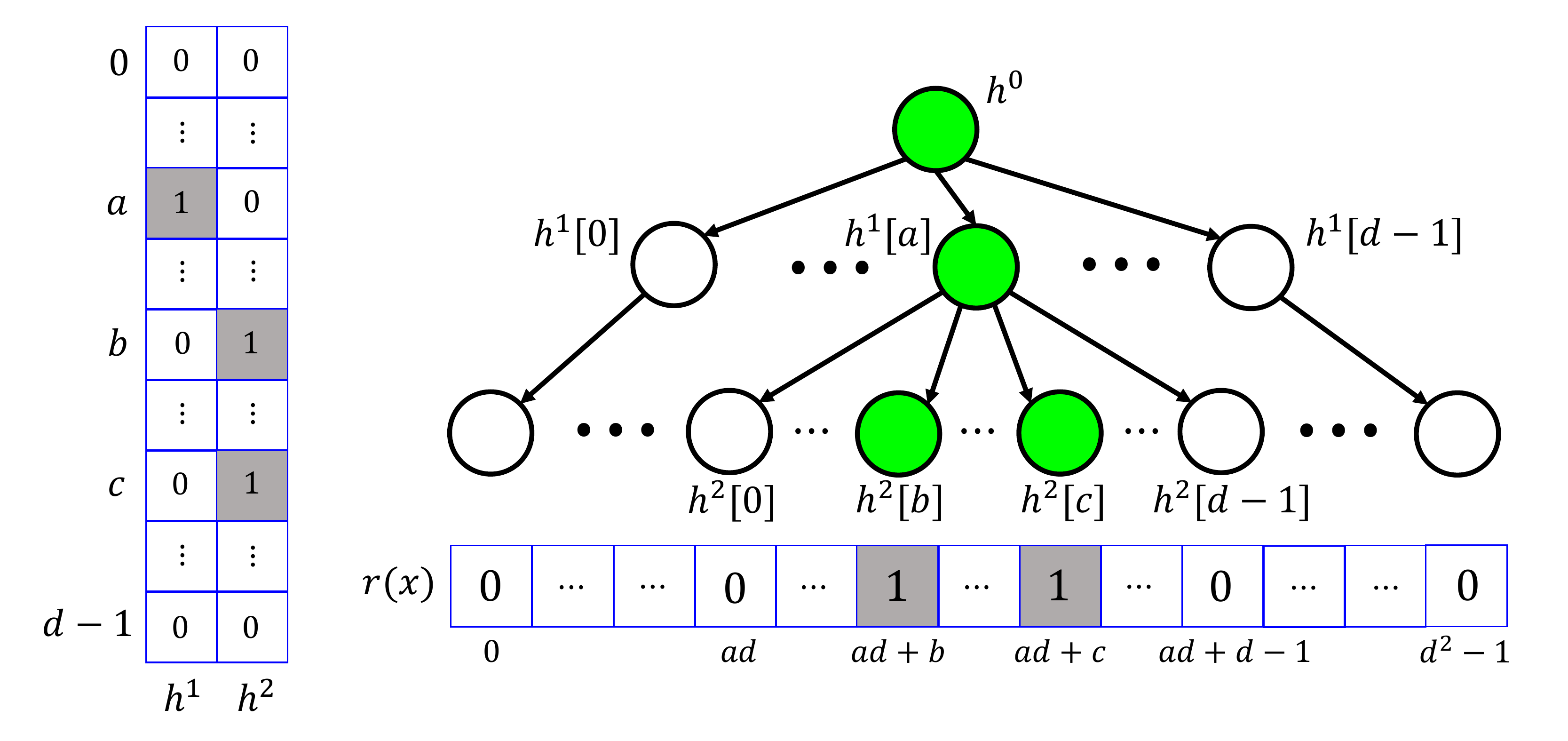}
\caption{Example hierarchical structure for $k\!=\!2$ and $k_s\!=\!2$. (Left) The hash code for each embedding representation $[f(\bfx_i; \bftheta)^1, f(\bfx_i; \bftheta)^2] \in \reals^{2d}$. (Right) Corresponding activated hash buckets out of total $d^2$ buckets.}
\label{fig:hash_tree}
\end{figure} \vspace{-0.3em}

\section{Methods}\vspace{-0.7em}
To compute the optimal set of embedding representations and the corresponding hash code,
the embedding representations are first required in order to infer which $k_s$ activations to set in the hash code, but to learn the embedding representations, it requires the hash code to determine which dimensions of the activations to adjust so that similar items would get hashed to the same buckets and vice versa. We take the alternating minimization approach iterating over computing the sparse hash codes respecting the hierarchical quantization structure and updating the network parameters indexed at the given hash codes per each mini-batch. \Cref{sec:method_hash} and \Cref{sec:method_embed} formalize the subproblems in detail. \vspace{-0.5em}

\subsection{Learning the hierarchical hash code} \vspace{-0.3em}
\label{sec:method_hash}
Given a set of continuous embedding representation $\{f(\bfx_i; \bftheta)\}_{i=1}^n$, we wish to compute the optimal binary hash code $\{\bfh_1, \ldots, \bfh_n\}$ so as to hash similar items to the same buckets and dissimilar items to different buckets. Furthermore, we seek to constrain the hash code to simultaneously maintain the hierarchical structure and the hard sparsity conditions throughout the optimization process. Suppose items $\bfx_i$ and $\bfx_j$ are dissimilar items, in order to hash the two items to different buckets, at each level of $\mathcal{T}$, we seek to encourage the hash code for each item at level $v$, $\bfh_i^v$ and $\bfh_j^v$ to differ. To achieve this, we optimize the hash code for all items per each level sequentially in cascading fashion starting from the first level $\{\bfh_1^1, \ldots, \bfh_n^1\}$ to the leaf nodes $\{\bfh_1^k, \ldots, \bfh_n^k\}$ as shown in \Cref{eqn:energy_seq}.
\footnotesize
\begin{align}
\label{eqn:energy_seq}
    &\minimize_{\bfh^k_{1:n},\ldots, \bfh^1_{1:n}}\quad \underbrace{\sum_{v=1}^{k} \sum_{i=1}^{n} -({f(\bfx_i; \bftheta)^{v}})^\intercal~  \bfh^v_i}_{\text{unary term}}\\
    &+ \underbrace{\sum_{v=2}^k \sum_{(i,j) \in \mathcal{N}} {\bfh^v_i}^\intercal Q' {\bfh^v_j}\prod_{w=1}^{v-1}\mathds{1}(\bfh^w_i=\bfh^w_j) }_{\text{sibling penalty}} + \underbrace{\sum_{v=1}^{k}\sum_{(i,j) \in \mathcal{N}} {\bfh_i^v}^\intercal P' {\bfh^v_j}}_{\text{orthogonality}} \nonumber\\
&\text{\ \  subject to }~~ \| \bfh^v_i \| = \begin{dcases}1 & \forall v\neq k\\ k_s &v = k \end{dcases},~ \bfh^v_i\in \{0,1\}^d,~ \forall i, \nonumber
\end{align}
%\begin{align}
%\label{eqn:energy_tree}
%    &\minimize_{\bfh^k_{1:n},\ldots, \bfh^1_{1:n}}\quad \underbrace{\sum_{v=1}^{k} \sum_{i=1}^{n} -({f(\bfx_i; \bftheta)^{v}})^\intercal~  \bfh^v_i}_{\text{unary term}}\nonumber\\
%    &\qquad\qquad\quad+ \underbrace{\sum_{v=2}^k \sum_{(i,j) \in \mathcal{N}} {\bfh^v_i}^\intercal Q' {\bfh^v_j}\prod_{w=1}^{v-1}\mathds{1}(\bfh^w_i=\bfh^w_j) }_{\text{sibling penalty}} \nonumber\\
%    &\qquad\qquad\quad+ \underbrace{\sum_{v=1}^{k}\sum_{(i,j) \in \mathcal{N}} {\bfh_i^v}^\intercal P' {\bfh^v_j}}_{\text{orthogonality}}\\
%&\text{\ \  subject to }~~ \| \bfh^v_i \| = \begin{dcases}1 & \forall v\neq k\\ k_s &v = k \end{dcases},~ \bfh^v_i\in \{0,1\}^d,~ \forall i, \nonumber
%\end{align}
\normalsize
where $\mathcal{N}$ denotes the set of dissimilar pairs of data and $\mathds{1}(\cdot)$ denotes the indicator function. Concretely, given the hash codes from \emph{all the previous levels}, we seek to minimize the following discrete optimization problem in \Cref{eqn:energy_tree}, subject to the same constraints as in \Cref{eqn:energy_seq}, sequentially for all levels\footnote{In \Cref{eqn:energy_tree}, we omit the dependence of $v$ for all $\bfh_1,\ldots,\bfh_n$ to avoid the notation clutter.} $v \in \{1,\ldots,k\}$. The unary term in the objective encourages selecting as large elements of each embedding vector as possible while the second term loops over \emph{all pairs of dissimilar siblings} and penalizes for their orthogonality. The last term encourages selecting as orthogonal elements as possible for a pair of hash codes from different classes in the current level $v$. The last term also makes sure, in the event that the second term becomes zero, the hash code still respects orthogonality among dissimilar items. This can occur when the hash code for all the previous levels was computed perfectly splitting dissimilar pairs into different branches and the second term becomes zero.
\vspace{-0.5em}
\scriptsize
\begin{align}
\label{eqn:energy_tree}
%&\minimize_{\bfh_1, \ldots, \bfh_n}\quad \underbrace{\sum_{i=1}^{n} -({f(\bfx_i; \bftheta)^{v}})^\intercal~  \bfh_i + \sum_{r=1}^l \sum_{p:id_p=r} \sum_{\substack{q:id_q=r\\y_q\neq y_p}}  \bfh_p^\intercal Q' \bfh_q + \sum_{i=1}^{n} \sum_{j:y_j \neq y_i} \bfh_i^\intercal P' \bfh_j}_{:=g(\bfh_1, \ldots, \bfh_n; \bftheta)} \nonumber\\
&\minimize_{\bfh_1, \ldots, \bfh_n}\quad \underbrace{\sum_{i=1}^{n} -({f(\bfx_i; \bftheta)^{v}})^\intercal~  \bfh_i}_{\text{unary term}} + \underbrace{\sum_{(i,j) \in \mathcal{S}^v}  \bfh_i^\intercal Q' \bfh_j}_{\text{sibling penalty}} + \underbrace{\sum_{(i,j) \in \mathcal{N}} \bfh_i^\intercal P' \bfh_j}_{\text{orthogonality}}
%%&\text{\ \  subject to }~~ \| \bfh_i \| = \begin{dcases}1 & \forall v\neq k\\ k_s &v = k \end{dcases},~ \bfh_i\in \{0,1\}^d,~ \forall i, \nonumber
\end{align}
\normalsize

\noindent where $\mathcal{S}^v = \left\{(i,j) \in \mathcal{N} \mid \bfh^w_i = \bfh^w_j, ~\forall w = 1,\ldots,v-1 \right\}$ denotes the set of pairs of siblings at level $v$ in $\mathcal{T}$, and $Q', P'$ encodes the pairwise cost for the sibling and the orthogonality terms respectively. However, optimizing \Cref{eqn:energy_tree} is NP-hard in general even in the simpler case of $k_s=1, k=1, d>2$ \cite{boykov_fast,jeong2018}. Inspired by \cite{jeong2018}, we use the average embedding of each class within the minibatch $\bfc_p^v = \frac{1}{m} \sum_{i:y_i = p} f(\bfx_i; \bftheta)^v \in \reals^{d}$ as shown in \Cref{eqn:energy_tree_avg}. 

\small
\begin{align}
&\minimize_{\bfz_1, \ldots, \bfz_{n_c}}\quad \underbrace{\sum_{p=1}^{n_c} -({\mathbf{c}_p^v})^\intercal \mathbf{z}_p +\sum_{\substack{(p,q) \in \mathcal{S}_z^v \\ p \neq q}} {\mathbf{z}_p}^\intercal Q \mathbf{z}_q + \sum_{p\neq q}  {\mathbf{z}_p}^\intercal P \mathbf{z}_q}_{:=\hat{g}(\bfz_1, \ldots, \bfz_{n_c})} \nonumber\\
&\text{\ \  subject to }~~ \| \mathbf{z}_p \| = \begin{dcases}1 & \forall v \neq k\\k_s & v=k\end{dcases},~ \mathbf{z}_p\in \{0,1\}^d,~ \forall p,
\label{eqn:energy_tree_avg}
\end{align}
\normalsize

\noindent where $\mathcal{S}_z^v=\left\{(p,q)\mid \bfz_p^w=\bfz_q^w, ~\forall w =1,\ldots,v-1  \right\}$, $n_c$ is the number of unique classes in the minibatch, and we assume each class has $m$ examples in the minibatch (\ie~ \textit{npairs} \cite{npairs} minibatch construction). Note, in accordance with the deep metric learning problem setting \cite{facenet,npairs,jeong2018}, we assume we are given access to the label adjacency information only within the minibatch. \\

\noindent The objective in \Cref{eqn:energy_tree_avg} upperbounds the objective in \Cref{eqn:energy_tree} (denote as $g(\cdot; \bftheta)$) by a gap $M(\bftheta)$ which depends only on $\bftheta$. Concretely, rewriting the summation in the unary term in $g$, we get
\small
\begin{align}
\label{eqn:bound}
&g(\bfh_1, \ldots, \bfh_n; \bftheta) = \sum_p^{n_c} \sum_{i:y_i=p} -({f(\bfx_i; \bftheta)^v})^\intercal~  \bfh_i\\
&\hspace{8em} + \sum_{(i,j)\in\mathcal{S}^v}  \bfh_i^\intercal Q' \bfh_j + \sum_{(i,j)\in\mathcal{N}} \bfh_i^\intercal P' \bfh_j \nonumber\\
&\leq\sum_p^{n_c} \sum_{i: y_i=p} -({\mathbf{c}_p^{v}})^\intercal~  \bfh_i + \sum_{(i,j)\in\mathcal{S}^v}  \bfh_i^\intercal Q' \bfh_j + \sum_{(i,j)\in\mathcal{N}} \bfh_i^\intercal P' \bfh_j \nonumber\\
&\quad +\underbrace{\maximize_{\hat\bfh_1, \ldots, \hat\bfh_n} \sum_p^{n_c} \sum_{i:y_i = p} (\bfc_p^{v} - f(\bfx_i; \bftheta)^{v})^\intercal~  \hat\bfh_i}_{:=M(\bftheta)}. \nonumber
\end{align}
\normalsize

\noindent Minimizing the upperbound in \Cref{eqn:bound} over $\bfh_1,\ldots,\bfh_n$ is identical to minimizing the objective $\hat{g}(\bfz_1, \ldots, \bfz_{n_c})$ in \Cref{eqn:energy_tree_avg} since each example $j$ in class $i$ shares the same class mean embedding vector $\bfc_i$. Absorbing the factor $m$ into the cost matrices \ie~ $Q=mQ'$ and $P = mP'$, we arrive at the upperbound minimization problem defined in \Cref{eqn:energy_tree_avg}. In the upperbound problem \Cref{eqn:energy_tree_avg}, we consider the case where the pairwise cost matrices are diagonal matrices of non-negative values. \Cref{thm:equivalence} in the following subsection proves that finding the optimal solution of the upperbound problem in \Cref{eqn:energy_tree_avg} is equivalent to finding the minimum cost flow solution of the flow network $G'$ illustrated in \Cref{fig:mcf}. Section B in the supplementary material shows the running time to compute the minimum cost flow (MCF) solution is approximately linear in $n_c$ and $d$. On average, it takes 24 \textit{ms} and 53 \textit{ms} to compute the MCF solution (discrete update) and to take a gradient descent step with npairs embedding \cite{npairs} (network update), respectively on a machine with 1 TITAN-XP GPU and Xeon E5-2650.

\subsection{Equivalence of the optimization problem to minimum cost flow} \vspace{-0.5em}
%%%%%%%%%%%%%%%%%%%%%%%%%%%%%%%%%
% Draw horizontal version
\begin{figure*}[ht]
\centering
\resizebox{2.0\columnwidth}{!}{
\begin{tikzpicture}[>=stealth',node distance=0.7cm]

\tikzstyle{vertex}=[circle,thick,draw,minimum size=0.8cm,inner sep=0.1pt]
\tikzstyle{group}=[inner sep=3.0pt,dotted,rounded corners,line width=1.5pt,draw=red]

% column 0
\node [vertex,fill=red!30,minimum size=0.8cm] (s) {$s$};
\node (in) [left=0.7cm of s,
	label={110:\texttt{Input flow}},
	label={190:$n_c k_s$}] {}
	edge[post] node {} (s);

% column 1
\node [vertex] (ap) [right= 1.5cm of s] {$a_p$}
	edge[pre] node[midway,fill=white] {$k_s, 0$} (s);
\node (ais) [above = 0.1cm of ap] {$\vdots$};
\node (aie) [below = 0.1cm of ap] {$\vdots$};
\node (ai) [group, draw=red,fit = (ais) (ap) (aie),label={[label distance=0.35cm]130:$A_r$}] {};

\node (a1) [above = 0.5cm of ais,label={[label distance=0.3cm]180:$A_{r-1}$}] {$\vdots$};
\draw[group] ([xshift=-0.35cm]a1.north west) -- ([xshift=-0.35cm,yshift=-0.2cm]a1.south west) -- ([xshift=0.35cm,yshift=-0.2cm]a1.south east) -- ([xshift=0.35cm]a1.north east);

\node (al) [below = 0.5cm of aie,label={[label distance=0.3cm]170:$A_{r+1}$}] {$\vdots$};
\draw[group] ([xshift=-0.35cm,yshift=-0.2cm]al.south west) -- ([xshift=-0.35cm,yshift=0.2cm]al.north west) -- ([xshift=0.35cm,yshift=0.2cm]al.north east) -- ([xshift=0.35cm,yshift=-0.2cm]al.south east);

% column 2
\node [vertex] (bi1) [above right = 0.65cm and 2.8cm of ap] {$b_{r,1}$}
	edge [pre] node[midway,fill=white] {$1, -c_p[0]$} (ap);
\node (bim1) [below = -0.3cm of bi1] {$\vdots$};
\node [vertex] (biq) [below = -0.1cm of bim1] {$b_{r,q}$}
	edge [pre] node[midway,fill=white] {$1, -c_p[q]$} (ap);
\node (bim2) [below = -0.3cm of biq] {$\vdots$};
\node [vertex] (bid) [below = -0.1 cm of bim2] {$b_{r,d}$}
	edge [pre] node[midway,fill=white] {$1, -c_p[d]$} (ap);
\node (bi) [group, draw=blue, fit = (bi1) (bid),label={[label distance=0.3cm]120:$B_r$}] {};

\node (b1) [above = 0.5cm of bi1,label={[label distance=0.35cm]183:$B_{r-1}$}] {$\vdots$};
\draw[group,draw=blue] ([xshift=-0.35cm,yshift=-0.3cm]b1.north west) -- ([xshift=-0.35cm,yshift=-0.2cm]b1.south west) -- ([xshift=0.35cm,yshift=-0.2cm]b1.south east) -- ([xshift=0.35cm,yshift=-0.3cm]b1.north east);

\node (bl) [below = 0.3cm of bid,label={[label distance=0.3cm]175:$B_{r+1}$}] {$\vdots$};
\draw[group,draw=blue] ([xshift=-0.35cm]bl.south west) -- ([xshift=-0.35cm,yshift=0.01cm]bl.north west) -- ([xshift=0.35cm,yshift=0.01cm]bl.north east) -- ([xshift=0.35cm]bl.south east);
\draw[thick, decoration={brace, mirror, raise=0.5cm}, decorate] ([xshift=0.5cm,yshift=-0.35cm ]al.east) -- ([xshift=2.7cm,yshift=-0.35cm]al.east);
\node (unary) [below right = 1cm and 2.75cm of al,label={[label distance=0.3cm]170:unary term}]{};
% column 3
\node [vertex] (b01) [right = 2.8 cm of bi1] {$b_{0,1}$}
	edge[pre,bend right=45] node {} (bi1)
	edge[pre,bend right=15] node {} (bi1);
\node (b0m1) [below = -0.3cm of b01] {$\vdots$};
\node [vertex] (b0q) [below = -0.1cm of b0m1] {$b_{0,q}$}
	edge[pre,bend right=17] node [midway,fill=white] {$1,2(g_r-1)\alpha$} (biq)
	edge[pre,bend left=17] node [midway,fill=white] {$1,0$} (biq);
\node (b0m2) [below = -0.3cm of b0q] {$\vdots$};
\node [vertex] (b0d) [below = -0.1cm of b0m2] {$b_{0,d}$}
	edge[pre,bend left=20] node {} (bid)
	edge[pre,bend left=50] node {} (bid);;
\node (b0) [group, draw=green, fit = (b01) (b0d),label={[label distance=0.01cm]90:$B_0$}] {};
\draw[thick, decoration={brace, mirror, raise=0.5cm}, decorate] ([xshift=4.0cm,yshift=-0.35cm ]al.east) -- ([xshift=6.3cm,yshift=-0.35cm]al.east);
\node (sibling) [below right = 0.9cm and 3.15cm of bl,label={[label distance=0.3cm]170:sibling penalty}]{};

% column 5
\node [vertex,fill=blue!30,minimum size=0.8cm] (t) [right = 2.9cm of b0q] {$t$}
	edge[pre,bend right=50] node {} (b01)
	edge[pre,bend right=25] node {} (b01)
	edge[pre,bend right=17] node [midway,fill=white] {$1,2(n_c-1)\beta$} (b0q)
	edge[pre,bend left=17] node [midway,fill=white] {$1,0$} (b0q)
	edge[pre,bend left=30] node {} (b0d)
	edge[pre,bend left=55] node {} (b0d);
\node (out) [right=0.7cm of t,
	label={30:\texttt{Output flow}},
	label={330:$n_c k_s$}] {}
	edge[pre] node {} (t);
\draw[thick, decoration={brace, mirror, raise=0.5cm}, decorate] ([xshift=7.5cm,yshift=-0.35cm ]al.east) -- ([xshift=10.3cm,yshift=-0.35cm]al.east);
\node (sibling) [below right = 1.85cm and 3.0cm of b0,label={[label distance=0.3cm]170:orthogonality}]{};
	
\path (biq) edge [draw=none,bend left=5,midway] node {$\vdots$} (b0q);
\path (biq) edge [draw=none,bend left=50,midway] node {$\vdots$} (b0q);
\path (bi1) edge [draw=none,bend left=35,midway] node {$\vdots$} (b01);
\path (biq) edge [draw=none,bend right=45,midway] node {$\vdots$} (b0q);
\path (bid) edge [draw=none,bend right=30,midway] node {$\vdots$} (b0d);

\path (b0q) edge [draw=none,bend left=5,midway] node {$\vdots$} (t);
\path (b0q) edge [draw=none,bend left=45,midway] node {$\vdots$} (t);
\path (b0q) edge [draw=none,bend right=40,midway] node {$\vdots$} (t);
\path (b0q) edge [draw=none,bend left=100,midway] node {$\vdots$} (t);
\path (b0q) edge [draw=none,bend right=85,midway] node {$\vdots$} (t);

\end{tikzpicture}
}
\vspace{-1em}
\caption{Equivalent flow network diagram $G'$ corresponding to the discrete optimization \Cref{eqn:energy_tree_avg}. Edge labels show the capacity and the cost respectively.} 
\label{fig:mcf} \vspace{-1em}
\end{figure*}
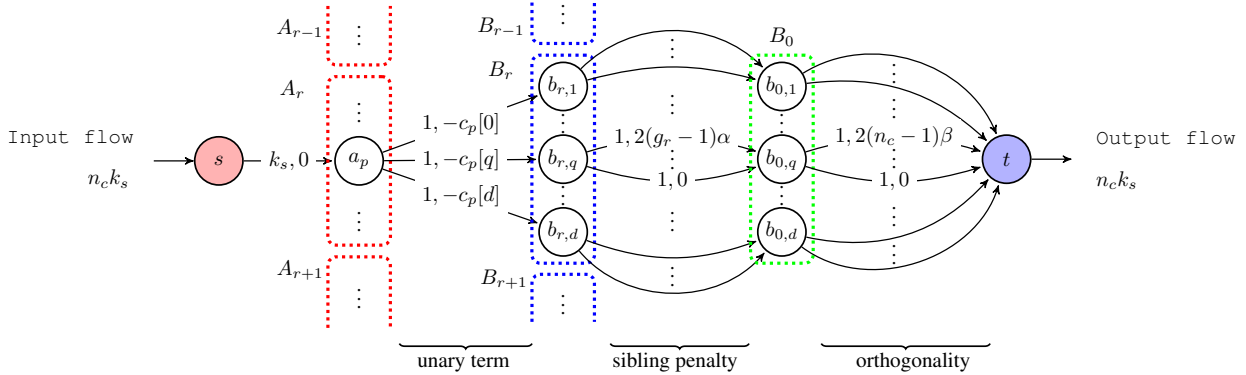
%%%%%%%%%%%%%%%%%%%%%%%%%%%%%%%%%

\begin{thm}\label{thm:equivalence}
The optimization problem in \Cref{eqn:energy_tree_avg} can be solved exactly in polynomial time by finding the minimum cost flow solution on the flow network G'.
\end{thm}
\vspace{-1em}
\begin{proof}
    Suppose we construct a vertex set $A = \{a_1,\ldots,a_{n_c}\}$ and partition $A$ into $\{A_r\}_{r=0}^l$ with the partition of $\{1,\ldots,n_c\}$ from equivalence relation $\mathcal{S}_{z}^v$\footnote{Define $(p,q) \in \mathcal{S}_z^v \iff a_p,a_q \in A_r, \forall r \geq 1$}. Here, we will define $A_0$ as a union of subsets of size $1$ (\ie~ each element in $A_0$ is a singleton without a sibling), and $A_1,\ldots,A_l$ as the rest of the subsets (of size greater than or equal to$2$). Concretely, $\left| A \right| = n_c$ and $A=\bigcup_{r=0}^l A_r$.

Then, we construct $l+1$ set of complete bipartite graphs $ \{G_r = \left( A_r \cup B_r, E_r \right) \}_{r=0}^l$ where we define $g_r\!=\!\left| A_r \right| $ and $\left| B_r \right| \!=\! d~~ \forall r$. Now suppose we construct a directed graph $G'$ by directing all edges $E_r$ from $A_r$ to $B_r$, attaching source $s$ to all vertices in $A_r$, and attaching sink $t$ to all vertices in $B_0$. Formally, $G' = \left(\bigcup_{r=0}^l \left(A_r \cup B_r\right) \cup \{s,t\}, E'\right)$. The edges in $E'$ inherit all directed edges from source to vertices in $A_r$, edges from vertices in $B_0$ to sink, and $\{E_r\}_{r=0}^l$. We also attach $g_r$ number of edges for each vertex $b_{r,q} \in B_r$ to $b_{0,q} \in B_0$ and attach $n_c$ number of edges from each vertex $b_{0,q} \in B_0$ to $t$. Concretely, $E'$ is 
\vspace{-0.5em}
\scriptsize
\begin{align}
    &\{(s,a_p)|a_p\in A\} \cup \bigcup_{r=0}^l E_r\cup \bigcup_{r=1}^l \{(b_{r,q}, b_{0,q})_i\}_{i=0}^{g_r-1} \cup \{(b_{0,q}, t)_j\}_{j=0}^{n_c-1}. \nonumber
\end{align}
\normalsize
\vspace{-1em}

Edges incident to $s$ have capacity $u(s,a_p) = k_s$ and cost $v(s, a_p)= 0$ for all $a_p \in A$. The edges between $a_p \in A_r$ and $b_{r,q} \in B_r$ have capacity $u(a_p, b_{r,q}) = 1$ and cost $v(a_p, b_{r,q}) = -\bfc_p[q]$. Each edge $i \in \{0, \ldots, g_r - 1\}$ between $b_{r,q} \in B_r$ and $b_{0,q} \in B_0$ has capacity $u\left( \left( b_{r,q}, b_{0,q} \right)_i \right) = 1$ and cost $u\left( \left( b_{r,q}, b_{0,q} \right)_i \right) = 2 \alpha i$. Each edge $j \in \{0, \ldots, n_c -1\}$ between $b_{0,q} \in B_0$ and $t$ has capacity $u\left( \left( b_{0,q}, t \right)_j \right) = 1$ and cost $v\left( \left( b_{0,q}, t \right)_j \right) = 2 \beta j$. Figure \ref{fig:mcf} illustrates the flow network $G'$. The amount of flow from source to sink is $n_c k_s$. The figure omits the vertices in $A_0$ and the corresponding edges to $B_0$ to avoid the clutter.\\

Now we define the flow $\{f_z(e)\}_{e \in E'}$ for each edge indexed both by flow configuration $\bfz_p \in \bfz_{1:n_c}$ where $\bfz_p \in \{0,1\}^d, \|\bfz_p\|_1 = k_s ~\forall p$ and $e \in E'$ below in \Cref{eqn:flow_def}.
\small
\begin{align}
&(i)~ f_z(s, a_p) = k_s,~ (ii)~ f_z(a_p, b_{r,q}) = \bfz_p[q]\nonumber\\
&(iii)~ f_z\left( \left( b_{r,q}, b_{0,q}\right)_i \right) = \begin{cases} 1 & ~\forall i < \sum_{p:a_p \in A_r} \mathbf{z}_p[q]\\ 0 & \text{otherwise}\end{cases} \nonumber\\
&(iv)~ f_z\left( \left(b_{0,q}, t\right)_j \right) = \begin{cases} 1 & ~\forall j < \sum_{p=1}^{n_c} \mathbf{z}_p[q]\\0 & \text{otherwise}\end{cases}
\label{eqn:flow_def}
\end{align}
\normalsize
To prove the equivalence of computing the minimum cost flow solution and finding the minimum binary assignment in \Cref{eqn:energy_tree_avg}, we need to show (1) that the flow defined in \Cref{eqn:flow_def} is feasible in $G'$ and (2) that the minimum cost flow solution of the network $G'$ and translating the computed flows to $\{\bfz_p\}$ in \Cref{eqn:energy_tree_avg} indeed minimizes the discrete optimization problem. We first proceed with the flow feasibility proof.\vspace{0.5em}

It is easy to see the capacity constraints are satisfied by construction in \Cref{eqn:flow_def} so we prove that the flow conservation conditions are met at each vertices. First, the output flow from the source $\sum_{a_p \in A} f_z(s,a_p) = \sum_{p=1}^{n_c} k_s = n_ck_s$ is equal to the input flow. For each vertex $a_p \in A$, the amount of input flow is $k_s$ and the output flow is the same $\sum_{b_{r,q}\in B_r} f_z(a_p,b_{r,q}) = \sum_{q=1}^d \mathbf{z}_p[q]=\| \mathbf{z}\|_1=k_s$.\\

For $r > 0$, for each vertex $b_{r,q} \in B_r$, denote the input flow as $y_{r,q} = \sum_{a_p \in A_r} f_z(a_p, b_{r,q}) = \sum_{p:a_p\in A_r} \mathbf{z}_p[q]$. The output flow is $\sum_{i=0}^{g_r-1} f_z((b_{r,q},b_{0,q})_i) = \sum_{p : a_p \in A_r} \bfz_p[q]= y_{r,q}$. The second term vanishes  because of \Cref{eqn:flow_def} (iii).  \vspace{0.5em}

The last flow conservation condition is to check the connections from each vertex $b_{0,q} \in B_0$ to the sink. Denote the input flow at the vertex as $y_{0,q} = \sum_{p:a_p\in A_0} \mathbf{z}_p[q] + \sum_{r=1}^l y_{r,q}= \sum_{p=1}^{n_c} \mathbf{z}_p[q]$. The output flow is $\sum_{j=0}^{n_c-1} f_z((b_{0,q}, t)_j) = \sum_{p=1}^{n_c} \bfz_p[q] = y_{0,q}$ which is identical to the input flow. Therefore, the flow construction in \Cref{eqn:flow_def} is feasible in $G'$.  \vspace{0.5em}

The second part of the proof is to check the optimality conditions and show the minimum cost flow finds the minimizer of \Cref{eqn:energy_tree_avg}. Denote, $\{f_o(e)\}_{e \in E'}$ as the minimum cost flow solution of the network $G'$ which minimizes the total cost $\sum_{e \in E'} v(e) f_o(e)$. Also denote the optimal flow from $a_p \in A_r$ to $b_{r,q} \in B_r,  f_o(a_p, b_q)$ as $\bfz'_p[q]$. By optimality of the flow, $\{f_o(e)\}_{e\in E'}$, $\sum_{e\in E'} v(e)f_o(e) \leq \sum_{e\in E'} v(e)f_z(e)~~ \forall z$. By \Cref{lemma:fo}, the \textit{lhs} of the inequality is equal to $\sum_{p=1}^{n_c} -{\mathbf{c}_p}^T\mathbf{z'}_p + \sum_{r=1}^l \sum_{p_1\neq p_2\in \{p|a_p\in A_r\}} \alpha {\mathbf{z'}_{p_1}}^T\mathbf{z'}_{p_2} + \sum_{p_1\neq p_2} \beta {\mathbf{z'}_{p_1}}^T\mathbf{z'}_{p_2}$. Additionally, \Cref{lemma:fz} shows the \textit{rhs} of the inequality is equal to $\sum_{p=1}^{n_c} -{\mathbf{c}_p}^T\mathbf{z}_p + \sum_{r=1}^l \sum_{p_1\neq p_2\in \{p|a_p\in A_r\}} \alpha {\mathbf{z}_{p_1}}^T\mathbf{z}_{p_2} + \sum_{p_1\neq p_2} \beta {\mathbf{z}_{p_1}}^T\mathbf{z}_{p_2}$.\vspace{0.5em} Finally, $\forall \{\bfz\}$
\scriptsize
\begin{align}
    &\sum_{p=1}^{n_c} -{\mathbf{c}_p}^T\mathbf{z'}_p + \sum_{r=1}^l \sum_{p_1\neq p_2\in \{p|a_p\in A_r\}} \alpha {\mathbf{z'}_{p_1}}^T\mathbf{z'}_{p_2} +\sum_{p_1\neq p_2} \beta {\mathbf{z'}_{p_1}}^T\mathbf{z'}_{p_2}\nonumber\\
    &\leq \sum_{p=1}^{n_c} -{\mathbf{c}_p}^T\mathbf{z}_p + \sum_{r=1}^l \sum_{p_1\neq p_2\in \{p|a_p\in A_r\}} \alpha {\mathbf{z}_{p_1}}^T\mathbf{z}_{p_2} + \sum_{p_1\neq p_2} \beta {\mathbf{z}_{p_1}}^T\mathbf{z}_{p_2}\nonumber.
\end{align}
\normalsize
This shows computing the minimum cost flow solution on $G'$ and converting the flows to $\bfz$'s, we can find the minimizer of the objective in \Cref{eqn:energy_tree_avg}.
\end{proof} \vspace{-1em}
\begin{lemma}\label{lemma:fo}
Given the minimum cost flow $\{f_o(e)\}_{e \in E'}$ of the network $G'$, the total cost of the flow is $\sum_{e\in E'} v(e)f_o(e) = \sum_{p=1}^{n_c} -{\mathbf{c}_p}^T\mathbf{z'}_p + \sum_{r=1}^l \sum_{p_1\neq p_2\in \{p|a_p\in A_r\}} \alpha {\mathbf{z'}_{p_1}}^T\mathbf{z'}_{p_2} + \sum_{p_1\neq p_2} \beta {\mathbf{z'}_{p_1}}^T\mathbf{z'}_{p_2}$.
\end{lemma}
\vspace{-1em}
\begin{proof}
    Proof in section A.2 of the supplementary material.
\end{proof}
\vspace{-1em}
\begin{lemma}\label{lemma:fz}
Given a feasible flow $\{f_z(e)\}_{e \in E'}$ of the network $G'$, the total cost of the flow is $\sum_{e\in E'} v(e)f_z(e) = \sum_{p=1}^{n_c} -{\mathbf{c}_p}^T\mathbf{z}_p + \sum_{r=1}^l \sum_{p_1\neq p_2\in \{p|a_p\in A_r\}} \alpha {\mathbf{z}_{p_1}}^T\mathbf{z}_{p_2} + \sum_{p_1\neq p_2} \beta {\mathbf{z}_{p_1}}^T\mathbf{z}_{p_2}$.
\end{lemma}
\vspace{-1em}
\begin{proof}
    Proof in section A.2 of the supplementary material.
\end{proof}
\vspace{-2em}
\subsection{Learning the embedding representation given the hierarchical hash codes} \vspace{-0.5em}
\label{sec:method_embed}
Given a set of binary hash codes for the mean embeddings $\{\bfz_1^v, \ldots, \bfz_{n_c}^v\}, ~\forall v=1, \ldots, k$ computed from \Cref{eqn:energy_tree_avg}, we can derive the hash codes for all $n$ examples in the minibatch, $\bfh_i^v := \bfz_p^v ~~\forall i : y_i = p$ and update the network weights $\bftheta$ given the hierarchical hash codes in turn. The task is to update the embedding representations, $\{f(\bfx_i; \bftheta)^v\}_{i=1}^n, ~\forall v=1, \ldots, k$, so that similar pairs of data have similar embedding representations indexed at the activated hash code dimensions and vice versa. Note, In terms of the hash code optimization in \Cref{eqn:energy_tree_avg} and the bound in \Cref{eqn:bound}, this embedding update has the effect of tightening the bound gap $M(\bftheta)$.

We employ the state of the art deep metric learning algorithms (denote as $\ell_{\text{metric}}(\cdot)$) such as \emph{triplet loss with semi-hard negative mining} \cite{facenet} and \emph{npairs loss} \cite{npairs} for this subproblem where the distance between two examples $\bfx_i$ and $\bfx_j$ at hierarchy level $v$ is defined as $d_{ij}^v = \| \left(\bfh_i^v \lor \bfh_j^v \right) \odot \left( f(\bfx_i; \bftheta)^v - f(\bfx_j; \bftheta)^v \right) \|_1$. Utilizing the logical \emph{OR} of the two binary masks, in contrast to independently indexing the representation with respective masks, to index the embedding representations helps prevent the pairwise distances frequently becoming zero due to the sparsity of the code. Note, this formulation in turn accommodates the backpropagation gradients to flow more easily. In our embedding representation learning subproblem, we need to learn the representations which respect the tree structural constraint on the corresponding hash code $\bfh = [\bfh^1, \ldots, \bfh^k] \in \{0,1\}^{d\times k}$ where $\| \bfh^v \|_1 = 1 ~~\forall v \neq k$ and $\| \bfh^k \|_1 = k_s$. To this end, we decompose the problem and compute the embedding loss per each hierarchy level $v$ separately. 

Furthermore, naively using the similarity labels to define similar pairs versus dissimilar pairs during the embedding learning subproblem could create a discrepancy between the hash code discrete optimization subproblem and the embedding learning subproblem leading to contradicting updates. Suppose two examples $\bfx_i$ and $\bfx_j$ are dissimilar and both had the highest activation at the same dimension $o$ and the hash code for some level $v$ was identical \ie~ $\bfh_i^v[o] = \bfh_j^v[o] = 1$. Enforcing the metric learning loss with the class labels, in this case, would lead to increasing the highest activation for one example and decreasing the highest activation for the other example. This can be problematic for the example with decreased activation because it might get hashed to another occupied bucket after the gradient update and this can repeat causing instability in the optimization process. 

However, if we relabel the two examples so that they are treated as the same class as long as they have the same hash code at the level, the update wouldn't decrease the activations for any example, and the sibling term (the second term) in \Cref{eqn:energy_tree_avg} would automatically take care of splitting the two examples in the next subsequent levels. 

To this extent, we apply \emph{label remapping} as follows. $y_i^v = remap(\bfh_i^v)$, where $remap(\cdot)$ assigns arbitrary unique labels to each unique configuration of $\bfh_i^v$. Concretely, $remap(\bfh_i^v) = remap(\bfh_j^v) \iff y_i^v = y_j^v$. Finally, the embedding representation learning subproblem aims to solve \Cref{eqn:embedding} given the hash codes and the remapped labels. Section C in the supplementary material includes the ablation study of label remapping.
\vspace{-0.5em}

\small
\begin{align}
\minimize_{\bftheta} \sum_{v=1}^k \ell_{\text{metric}} \left( \{f(x_i; \bftheta)^v\}_{i=1}^n; ~\{\bfh_i^v\}_{i=1}^n, \{y_i^v\}_{i=1}^n \right)
\label{eqn:embedding}
\end{align}
\normalsize
Following the protocol in \cite{jeong2018}, we use the Tensorflow implementation of deep metric learning algorithms in \href{https://www.tensorflow.org/versions/master/api_docs/python/tf/contrib/losses/metric_learning}{\texttt{tf.contrib.losses.metric\_learning}}.\vspace{-0.5em}

\section{Implementation details} \vspace{-1em}
\begin{algorithm}[H]
   \caption{Learning algorithm}
   \label{alg:procedure}
\begin{algorithmic}
                  \INPUT $\bftheta_b^{\text{emb}}$ (pretrained metric learning base model); $\bftheta_d$, $k$
                  \REQUIRE $\bftheta_f = [\bftheta_b, \bftheta_d]$
                  \FOR{ $t=1,\ldots,$ MAXITER}
                    \STATE Sample a minibatch $\{\bfx_i\}$ and initialize $\mathcal{S}^1_{z}=\emptyset$
                    \FOR{$v=1,\cdots, k$}
                      \STATE Update the flow network $G'$ by computing class cost vectors\\
                      $~~~~\bfc_p^v = \frac{1}{m} \sum_{i:y_i = p} f(\bfx_i; \bftheta_f)^v$
                      \STATE Compute the hash codes $\{\bfh^{v}_i\}$ via minimum cost flow on $G'$
                      \STATE Update $\mathcal{S}^{v+1}_{z}$ given $\mathcal{S}^v_{z}$ and $\{h_i^v\}$
                      \STATE Remap the label to compute $y^v$
                    \ENDFOR
                    \STATE Update the network parameter given the hash codes
                        \vspace{-0.5em}
                        $$\bftheta_f \leftarrow \bftheta_f - \eta^{(t)} \partial_{\bftheta_f} \sum_{v=1}^k \ell_\text{metric}(\bftheta_f; ~\bfh_{1:n_c}^v, y_{1:n_c}^v) $$ \vspace{-1em}
                    \STATE Update stepsize $\eta^{(t)} \leftarrow$ ADAM rule \cite{adam}
                   \ENDFOR
               \OUTPUT $\bftheta_{f}$ (final estimate);
\end{algorithmic}
\end{algorithm}

\label{sec:implementation}
\textbf{Network architecture}~ For fair comparison, we follow the protocol in \cite{jeong2018} and use the NIN \cite{NIN} architecture (denote the parameters $\bftheta_b$) with \emph{leaky relu} \cite{xu2015empirical} with $\tau=5.5$ as activation function and train Triplet embedding network with semi-hard negative mining \cite{facenet}, Npairs network \cite{npairs} from scratch as the base model, and snapshot the network weights ($\bftheta_b^{\text{emb}}$) of the learned base model. Then we replace the last layer in ($\bftheta_b^{\text{emb}}$) with a randomly initialized $dk$ dimensional fully connected projection layer ($\bftheta_d$) and finetune the hash network (denote the parameters as $\bftheta_f = [\bftheta_b, \bftheta_d]$). \Cref{alg:procedure} summarizes the learning procedure.\vspace{-0.2em}

\textbf{Hash table construction and query}~ We use the learned hash network $\bftheta_f$ and apply \Cref{eqn:hash_function} to convert $\bfx_i$ into the hash code $\bfh(\bfx_i; \bftheta_f)$ and use the base embedding network $\bftheta_b^{\text{emb}}$ to convert the data into the embedding representation $f(\bfx_i; \bftheta_b^{\text{emb}})$. Then, the embedding representation is hashed to buckets corresponding to the $k_s$ set bits in the hash code. During inference, we convert a query data $\bfx_q$ into the hash code $\bfh(\bfx_q; \bftheta_f)$ and into the embedding representation $f(\bfx_q; \bftheta_b^{\text{emb}})$.  Once we retrieve the union of all bucket items indexed at the $k_s$ set bits in the hash code, we apply a reranking procedure \cite{survey_learningtohash} based on the euclidean distance in the embedding space.\vspace{-0.2em} 

\textbf{Evaluation metrics}~ Following the evaluation protocol in \cite{jeong2018}, we report our accuracy results using precision@k (Pr@k) and normalized mutual information (NMI) \cite{manningbook} metrics. Precision@k is computed based on the reranked ordering (described above) of the retrieved items from the hash table. We evaluate NMI, when the code sparsity is set to $k_s=1$, treating each bucket as an individual cluster. We report the speedup results by comparing the number of retrieved items versus the total number of data (exhaustive linear search) and denote this metric as SUF. \vspace{-0.7em}

\section{Experiments}
\label{sec:exp}
\begin{table*}[ht]
\centering
\fontsize{6pt}{6.5pt}\selectfont
\begin{tabular}{cc cccccccc cccccccc}
    \addlinespace[-\aboverulesep]
    \cmidrule[1pt](r){1-10} \cmidrule[1pt](r){11-18}
    &      &  \multicolumn{8}{c}{Triplet}                                                                                              &\multicolumn{8}{c}{Npairs} \\
    &      &  \multicolumn{4}{c}{\emph{test}}                                & \multicolumn{4}{c}{\emph{train}}                        &\multicolumn{4}{c}{\emph{test}}                                & \multicolumn{4}{c}{\emph{train}} \\
    \cmidrule(r){1-6} \cmidrule(r){7-10} \cmidrule(r){11-14} \cmidrule(r){15-18}
    &Method& SUF            & Pr@1           & Pr@4          & Pr@16         & SUF            & Pr@1           & Pr@4          & Pr@16 &SUF            & Pr@1           & Pr@4          & Pr@16         & SUF            & Pr@1           & Pr@4          & Pr@16 \\
    \cmidrule(r){1-6} \cmidrule(r){7-10} \cmidrule(r){11-14} \cmidrule(r){15-18}
    $k_s$
    &Metric& 1.00          &  56.78         & 55.99         & 53.95         & 1.00           &  62.64         & 61.91         & 61.22  &1.00            & 57.05          & 55.70         & 53.91         & 1.00           &61.78           & 60.63         & 59.73 \\
    \cmidrule(r){1-6} \cmidrule(r){7-10} \cmidrule(r){11-14} \cmidrule(r){15-18}
    $1$
    &LSH    & \textbf{138.83}&52.52 & 48.67&39.71&\textbf{135.64} & 60.45 &58.10&54.00 & 29.74&53.55&50.75&43.03&30.75&59.87&58.34&55.35\\
    &DCH    & 96.13         &  56.26         & 55.65         & 54.26         & 89.60          &  61.06         & 60.80         & 60.81 &41.59          &  57.23         & 56.25         & 54.45         & 40.49          & 61.59          & 60.77         & 60.12 \\
    &Th     & 41.21         &  54.82         & 52.88         & 48.03         & 43.19          &  61.56         & 60.24         & 58.23 &12.72          &  54.95         & 52.60         & 47.16         & 13.65          & 60.80          & 59.49         & 57.27 \\  
    &VQ     & 22.78         &  56.74         & 55.94         & 53.77         & 40.35          & 62.54          & 61.78         & 60.98 &34.86          & 56.76          & 55.35         & 53.75         & 31.35          &  61.22         & 60.24         & 59.34 \\ 
    &\cite{jeong2018}    & 97.67&  57.63         & 57.16         & 55.76         & 97.77 &  63.85         & 63.40 & 63.39&54.85          & 58.19 & 57.22         & 55.87         & 54.90          &  \textbf{63.11}& 62.29         & 61.94 \\
    &Ours   & 97.67&  \textbf{58.42}& \textbf{57.88}& \textbf{56.58}& 97.28          &  \textbf{64.73}& \textbf{64.63}& \textbf{64.69}&\textbf{101.1}& \textbf{58.28}  & \textbf{57.79}& \textbf{56.92}& \textbf{97.47} &  63.06         & \textbf{62.62}& \textbf{62.44} \\
    \cmidrule(r){1-6} \cmidrule(r){7-10} \cmidrule(r){11-14} \cmidrule(r){15-18}
    $2$
    &Th     & 14.82         &  56.55         & 55.62         & 52.90         & 15.34          &  62.41         & 61.68         & 60.89&5.09           &  56.52         & 55.28         & 53.04         & 5.36           &  61.65         & 60.50         & 59.50  \\ 
    &VQ     & 5.63          &  56.78         & 56.00         & 53.99         & 6.94           &  62.66         & 61.92         & 61.26&6.08           &  57.13         & 55.74         & 53.90         & 5.44           &  61.82         & 60.56         & 59.70  \\ 
    &\cite{jeong2018}   & 76.12         &  57.30         & 56.70         & 55.19         & 78.28          & 63.60      & 63.19 & 63.09&16.20          &  57.27         & 55.98         & 54.42         & 16.51          &  61.98         & 60.93         & 60.15 \\
    &Ours   & \textbf{98.38}&  \textbf{58.39}& \textbf{57.51}& \textbf{56.09}& \textbf{97.20} &  \textbf{64.35}& \textbf{63.91}& \textbf{63.81}& \textbf{69.48} &  \textbf{57.60}& \textbf{56.98}& \textbf{55.82}& \textbf{69.91} &  \textbf{62.19}& \textbf{61.71}& \textbf{61.27} \\
    \cmidrule(r){1-6} \cmidrule(r){7-10} \cmidrule(r){11-14} \cmidrule(r){15-18}
    $3$
    &Th     & 7.84          &  56.78         & 55.91         & 53.64         & 8.04           &  62.66         & 61.88         & 61.16& 3.10           &  56.97         & 55.56         & 53.76         & 3.21           &  61.75         & 60.66         & 59.73 \\
    &VQ     & 2.83          &  56.78         & 55.99         & 53.95         & 2.96           &  62.62         & 61.92         & 61.22& 2.66           &  57.01         & 55.69         & 53.90         & 2.36           &  61.78         & 60.62         & 59.73 \\ 
    &\cite{jeong2018}   & 42.12         &  56.97         & 56.25         & 54.40         & 44.36          &  62.87     & 62.22 & 61.84& 7.25           &  57.15         & 55.81         & 54.10         & 7.32           &  61.90         & 60.80         & 59.96 \\ 
    &Ours   & \textbf{94.55}& \textbf{58.19} & \textbf{57.42}& \textbf{56.02}& \textbf{93.69}&\textbf{63.60}&\textbf{63.35}&\textbf{63.32}& \textbf{57.09}& \textbf{57.56}& \textbf{56.70}& \textbf{55.41}& \textbf{58.62} &  \textbf{62.30}& \textbf{61.44}& \textbf{60.91} \\
    \cmidrule(r){1-6} \cmidrule(r){7-10} \cmidrule(r){11-14} \cmidrule(r){15-18}
    $4$
    &Th     & 4.90          & 56.84          & 56.01         & 53.86         & 5.00           &  62.66         & 61.94         & 61.24 & 2.25           &  57.02         & 55.64         & 53.88         & 2.30           & 61.78          & 60.66         & 59.75 \\ 
    &VQ     & 1.91          & 56.77          & 55.99         & 53.94         & 1.97           & 62.62          & 61.91         & 61.22 & 1.66           & 57.03          & 55.70         & 53.91         & 1.55           &  61.78         & 60.62         & 59.73 \\  
    &\cite{jeong2018}    & 16.19         &  57.11         & 56.21         & 54.20         & 16.52          &  62.81         & 62.14         & 61.58 & 4.51           & 57.15          & 55.77         & 54.01         & 4.52           & 61.81          & 60.69         & 59.77 \\ 
    &Ours   & \textbf{92.18}&  \textbf{58.52}& \textbf{57.79}& \textbf{56.22}& \textbf{91.27} &  \textbf{64.20}& \textbf{63.95}& \textbf{63.63} & \textbf{49.43} & \textbf{57.75} & \textbf{56.79}& \textbf{55.50}& \textbf{50.80} & \textbf{62.43} & \textbf{61.65}& \textbf{61.01} \\
    \addlinespace[-\belowrulesep]
    \cmidrule[1pt](r){1-6} \cmidrule[1pt](r){7-10} \cmidrule[1pt](r){11-14} \cmidrule[1pt](r){15-18}
\end{tabular}
\caption{Results with Triplet network with hard negative mining and Npairs network. Querying test data against a hash table built on \emph{test} set and a hash table built on \emph{train} set on Cifar-100.}
\label{tab:cifar}
\end{table*}

\begin{table}[ht]
\setlength{\tabcolsep}{4pt}
\centering
\fontsize{6pt}{6.5pt}\selectfont
\begin{tabular}{cccccccccc}
\addlinespace[-\aboverulesep]
\cmidrule[1pt](r){1-6} \cmidrule[1pt](r){7-10}
&      &  \multicolumn{4}{c}{Triplet}   &\multicolumn{4}{c}{Npairs} \\
\cmidrule(r){1-6} \cmidrule(r){7-10}
&Method & SUF         &  Pr@1        &      Pr@4    &        Pr@16 & SUF           &  Pr@1        &      Pr@4    &        Pr@16 \\ 
\cmidrule(r){1-6} \cmidrule(r){7-10}
$k_s$
&Metric& 1.00       &   10.90 & 9.39  &  7.45 & 1.00          &   15.73      &    13.75     &         11.08\\
\cmidrule(r){1-6} \cmidrule(r){7-10}
$1$
& LSH &           164.25&  8.86&  7.23&  5.04&     112.31      &         11.71&         8.98&          5.56\\
& DCH &          140.77&  9.82&  8.43&  6.44&      220.52&       13.87&              11.77&          8.99\\
& Th &           18.81&  10.20&  8.58&  6.50&     1.74      &         15.06&         12.92&          9.92\\
& VQ &          146.26&  10.37& 8.84& 6.90&    451.42     &         15.20&      	 13.27&         10.96\\
& \cite{jeong2018}&         221.49&    \textbf{11.00}& \textbf{9.59}& 7.83&    478.46     &16.95&         15.27&         13.06\\ 
& Ours&\textbf{590.41}&         10.91&      9.58& \textbf{7.85} &\textbf{952.49}&   \textbf{17.00}&\textbf{15.53}&\textbf{13.54}\\
\cmidrule(r){1-6} \cmidrule(r){7-10}
$2$
&Th  &           6.33&      10.82&            9.30&         7.32&            1.18&       15.70&          13.69&              10.96\\
&VQ  &          32.83&      10.88&            9.33&         7.39&          116.26&       15.62&          13.68&              11.15\\
&\cite{jeong2018}&          60.25&      11.10&            9.64&         7.73&          116.61&       16.40&          14.49&             12.00\\ 
&Ours&\textbf{533.86}&\textbf{11.14}&\textbf{9.72}&\textbf{7.96}&\textbf{1174.35}&  \textbf{17.22}&\textbf{15.57}&  \textbf{13.63}\\
\cmidrule(r){1-6} \cmidrule(r){7-10}
$3$
&Th&                 3.64&          10.87&          9.38&   7.42&            1.07&         15.73&         13.74&        11.07\\
&VQ&                13.85&          10.90&          9.38&   7.44&           55.80&         15.74&         13.74&        11.12\\
&\cite{jeong2018}&              27.16& 11.20&           9.55&   7.60&           53.98&        16.24&         14.32&        11.73\\ 
&Ours&    \textbf{477.86}&          \textbf{11.21}&\textbf{9.72}&\textbf{7.94}&\textbf{1297.98}&\textbf{17.09}&\textbf{15.37}&\textbf{13.39}\\
\addlinespace[-\belowrulesep]
\cmidrule[1pt](r){1-6} \cmidrule[1pt](r){7-10}
\end{tabular}
\caption{Results with Triplet network with hard negative mining and Npairs \cite{npairs}  Network. Querying ImageNet \emph{val} data against hash table built on \emph{val} set.}
\label{tab:imagenet} 
\end{table}
We report our results on Cifar-100 \cite{cifar100} and ImageNet \cite{imagenet} datasets and compare against several baseline methods. First baseline methods are the state of the art deep metric learning models \cite{facenet, npairs} performing an exhaustive linear search over the whole dataset given a query data (denote as `Metric'). Next baseline is the Binarization transform \cite{agrawal2014,zhai2017} where the dimensions of the hash code corresponding to the top $k_s$ dimensions of the embedding representation are set (denote as `Th'). Then we perform vector quantization \cite{survey_learningtohash} on the learned embedding representation from the deep metric learning methods above on the entire dataset and compute the hash code based on the indices of the $k_s$ nearest centroids (denote as `VQ'). Another baseline is the quantizable representation in \cite{jeong2018}(denote as \cite{jeong2018}). In both Cfar-100 and ImageNet, we follow the data augmentation and preprocessing steps in \cite{jeong2018} and train the metric learning base model with the same settings in \cite{jeong2018} for fair comparison. In Cifar-100 experiment, we set $(d,k)=(32,2)$ and $(d,k)=(128,2)$ for the npairs network and the triplet network, respectively. In ImageNet experiment, we set $(d,k)=(512,2)$ and $(d,k)=(256,2)$ for the npairs network and the triplet network, respectively. In ImageNetSplit experiment, we set $(d,k)=(64,2)$. We also perform LSH hashing \cite{jain2008fast} baseline and Deep Cauchy Hashing \cite{dch} baseline which both generate $n$-bit binary hash codes with $2^n$ buckets and compare against other methods when $k_s\!=\!1$ (denote as `LSH' and `DCH', respectively). For the fair comparison, we set the number of buckets, $2^n\!=\!dk$. 

\subsection{Cifar-100}
Cifar-100 \cite{cifar100} dataset has $100$ classes. Each class has $500$ images for \emph{train} and $100$ images for \emph{test}.  Given a query image from \emph{test}, we experiment the search performance both when the hash table is constructed from \emph{train} and from \emph{test}.  The batch size is set to $128$ in Cifar-100 experiment. We finetune the base model for $70$k iterations and decayed the learning rate to $0.3$ of previous learning rate after $20$k iterations when we optimize our methods.  \Cref{tab:cifar} shows the results from the triplet network and the npairs network respectively.  The results show that our method not only outperforms search accuracies of the state of the art deep metric learning base models but also provides the superior speedup over other baselines. \vspace{-0.7em}

\begin{table}[ht]
\centering
\fontsize{6pt}{6.5pt}\selectfont
\begin{tabular}{c ccc ccc}
\cmidrule[1pt](r){1-4} \cmidrule[1pt](r){5-7}
&  \multicolumn{3}{c}{Triplet}   &\multicolumn{3}{c}{Npairs} \\
\cmidrule(r){1-4} \cmidrule(r){5-7}
& \multicolumn{2}{c}{Cifar-100} & ImageNet &\multicolumn{2}{c}{Cifar-100} & ImageNet \\ 
& train         & test          & val      &train         & test          & val      \\ 
\cmidrule(r){1-4} \cmidrule(r){5-7}
LSH   & 62.94         & 53.11         &  37.90   & 43.80         & 37.45             &  36.00 \\
DCH   & 86.11         & 68.88         & 45.55   & 80.74         & 65.62            & 50.01  \\ 
Th   & 68.20         & 54.95         & 31.62    & 51.46         & 44.32             & 15.20    \\
VQ   & 76.85         & 62.68         & 45.47    & 80.25         & 66.69             & 53.74    \\
\cite{jeong2018} & 89.11         & 68.95         & 48.52     & 84.90& 68.56             & 55.09    \\
Ours & \textbf{89.95}   & \textbf{69.64}& \textbf{61.21} & \textbf{86.80}         & \textbf{71.30}    & \textbf{65.49}   \\ 
\cmidrule[1pt](r){1-4} \cmidrule[1pt](r){5-7}
\end{tabular}
\caption{Hash table NMI for Cifar-100 and Imagenet.}
\label{tab:NMI} 
\end{table}

\begin{table}[ht]
\centering
\fontsize{6pt}{6.5pt}\selectfont
\begin{tabular}{cccccccccc}
\addlinespace[-\aboverulesep]
\toprule
&Method & SUF         &  Pr@1        &      Pr@4    &        Pr@16 \\
\midrule
$k_s$
&Metric & 1.00        &   21.55      &      19.11    &       16.06 \\
\midrule
$1$
&LSH    & 33.75       &  18.49       &       15.50    & 11.14          \\
& Th    & 10.98       &  20.25       &       17.22    &  13.66        \\
& VQ-train& 54.30     &  20.15       &       18.10    &  14.85        \\
& VQ-test& 57.44       &  20.59       &      18.31    &  15.32        \\
& \cite{jeong2018}& 56.35&  21.35       & 18.49    &  15.32    \\
& Ours    & \textbf{78.23}&  \textbf{21.46} & \textbf{18.88}    &  \textbf{15.67} \\
\midrule
$2$
& Th    & 4.55       &  21.27       &   18.86    &  15.68        \\
& VQ-train & 15.29       &  21.51 &  19.03    &  15.88        \\
& VQ-test & 16.43       &  21.58  &  18.93    &  15.94        \\
& \cite{jeong2018}    & 15.99       &  \textbf{22.12}       & \textbf{19.21}   &  \textbf{15.95}        \\
& Ours    & \textbf{71.14}       &  \textbf{22.12}       & 18.63    &  15.34        \\
\midrule
$3$
& Th    & 2.79       &  21.53       & 19.11    &  15.99      \\
& VQ-train & 7.80      &  21.56      &      19.11    &  16.03        \\
& VQ-test  & 8.20       &  21.58       &    19.09    &  16.06        \\
& \cite{jeong2018}    & 7.24      &  \textbf{22.18}       &       \textbf{19.40}    &  \textbf{16.10} \\
& Ours    & \textbf{84.04} &21.97       & 18.87    &  15.56        \\
\addlinespace[-\belowrulesep]
\bottomrule
\end{tabular}
\caption{Results with Triplet network with hard negative mining. Querying ImageNet \emph{val} set in $C_{\text{test}}$ against hash table built on \emph{val} set in $C_{\text{test}}$.}
\label{tab:imagenetsplit} 
\end{table}

\begin{figure*}[ht]
\centering
\includegraphics[width=\linewidth]{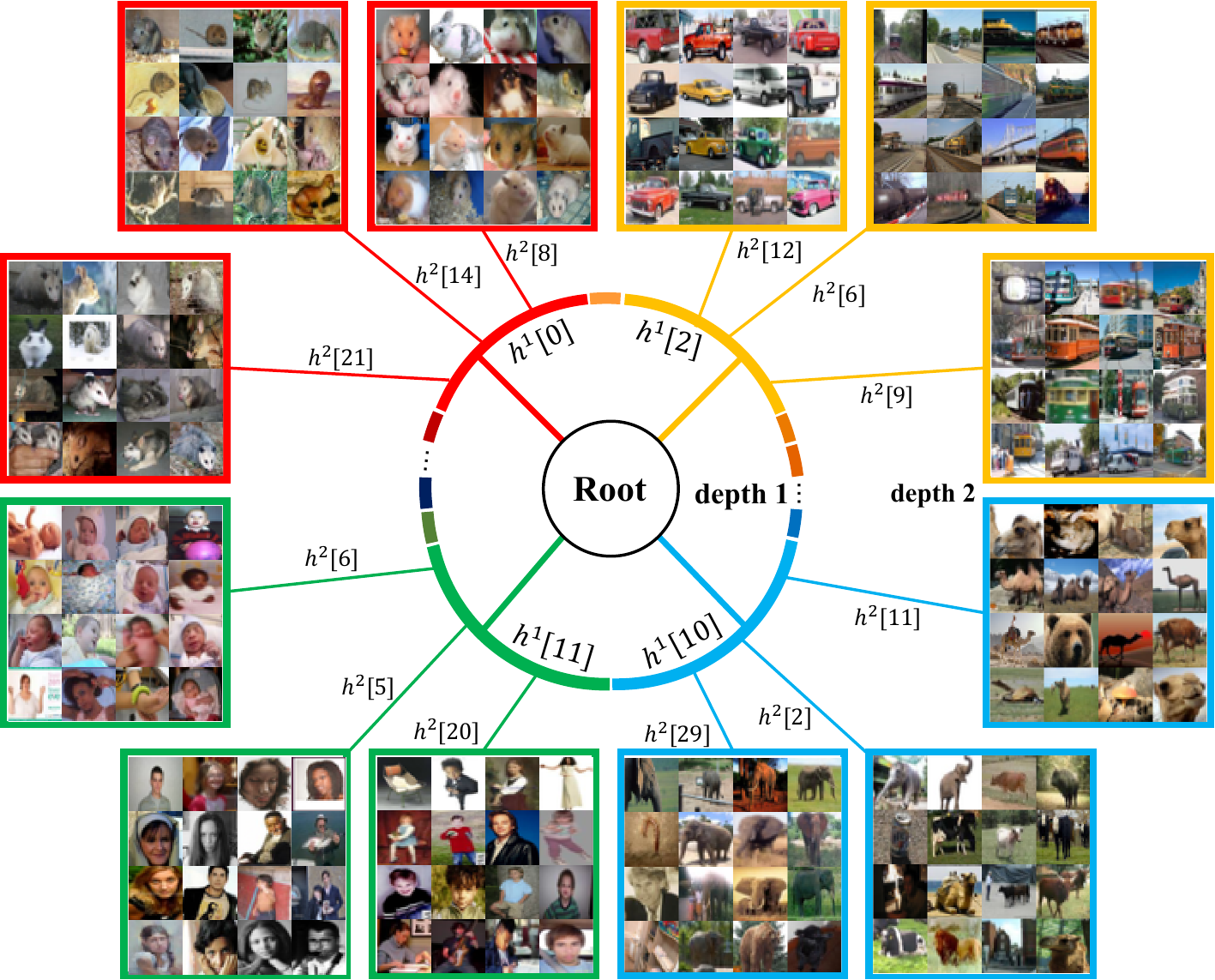}
    \caption{Visualization of the examples mapped by our trained three level hash codes $[\bfh^{(1)}, \bfh^{(2)}]$ on Cifar-100. Each parent node (denoted as depth 1) is color coded in red, yellow, blue, and green in \emph{cw} order. Each color coded box (denoted as depth 2) shows examples of the hashed items in each child node.}
    \label{fig:cifarhashtree}
\end{figure*} 

\subsection{ImageNet}
ImageNet ILSVRC-2012 \cite{imagenet} dataset has $1,000$ classes and comes with \emph{train} ($1,281,167$ images) and \emph{val} set ($50,000$ images).  We use the first nine splits of \emph{train} set to train our model, the last split of \emph{train} set for validation, and use \emph{validation} dataset to test the query performance.  We use the images downsampled to $32\times32$ from \cite{imgnet-down}. We finetune npairs base model and triplet base model as in \cite{jeong2018} and add a randomly initialized fully connected layer to learn hierarchical representation.  Then, we train the parameters in the newly added layer with other parameters fixed. When we train with npairs loss, we set the batch size to $1024$ and train for $15$k iterations decaying the learning rate to $0.3$ of previous learning rate after each $6$k iterations.  Also, when we train with triplet loss, we set the batch size to $512$ and train for $30$k iterations decaying the learning rate of $0.3$ of previous learning rate after each $10$k iterations.  Our results in \Cref{tab:imagenet} show that our method outperforms the state of the art deep metric learning base models in search accuracy while providing up to $1298\times$ speedup over exhaustive linear search. \Cref{tab:NMI} compares the NMI metric and shows that the hash table constructed from our representation yields buckets with significantly better class purity on both datasets and on both the base metric learning methods.

\subsection{ImageNetSplit}
In order to test the generalization performance of our learned  representation against previously unseen classes, we performed an experiment on ImageNet where the set of classes for training and testing are completely disjoint. Each class in ImageNet ILSVRC-2012 \cite{imagenet} dataset has super-class based on WordNet \cite{wordnet}. We select $119$ super-classes which have exactly two sub-classes in $1000$ classes of ImageNet ILSVRC-2012 dataset. Then, we split the two sub-classes of each $119$ super-class into $C_{\text{train}}$ and $C_{\text{test}}$, where $C_{\text{train}}\cap C_{\text{test}}=\emptyset$. Section D in the supplementary material shows the class names in $C_{\text{train}}$ and $C_{\text{test}}$. We use the images downsampled to $32\times32$ from \cite{imgnet-down}. We train the models with triplet embedding on $C_{\text{train}}$ and test the models on $C_{\text{test}}$. The batch size is set to $200$ in ImageNetSplit dataset. We finetune the base model for $50$k iterations and decayed the learning rate to $0.3$ of previous learning rate after $40$k iterations when we optimize our methods. We also perform vector quantization with the centroids obtained from $C_{\text{train}}$ (denote as `VQ-train') and $C_{\text{test}}$ (denote as `VQ-test'), respectively. \Cref{tab:imagenetsplit} shows our method preserves the accuracy without compromising the speedup factor.\\
\vspace{-1em}

Note, in all our experiments in \Cref{tab:cifar,tab:imagenet,tab:NMI,tab:imagenetsplit}, while all the baseline methods show severe degradation in the speedup over the code compound parameter $k_s$, the results show that the proposed method robustly withstands the speedup degradation over $k_s$. This is because our method 1) greatly increases the quantization granularity beyond other baseline methods and 2) hashes the items more uniformly over the buckets. In effect, indexing multiple buckets in our quantized representation does not as adversarially effect the search speedup as other baselines. \Cref{fig:cifarhashtree} shows a qualitative result with npairs network on Cifar-100, where $d=32, k=2, k_s=1$. As an interesting side effect, our qualitative result indicates that even though our method does not use any super/sub-class labels or the entire label information during training, optimizing for the objective in \Cref{eqn:energy_seq} naturally discovers and organizes the data exhibiting a meaningful hierarchy where similar subclasses share common parent nodes.

\section{Conclusion} 
We have shown a novel end-to-end learning algorithm where the quantization granularity is significantly increased via hierarchically quantized representations while preserving the search accuracy and maintaining the computational complexity practical for the mini-batch stochastic gradient descent setting. This not only provides the state of the art accuracy results but also unlocks significant improvement in inference speedup providing the highest reported inference speedup on Cifar100 and ImageNet datasets respectively.

\section*{Acknowledgements}
%We would like to thank the anonymous reviewers for their constructive comments.
This work was partially supported by Kakao, Kakao Brain and Basic Science Research Program through the National Research Foundation of Korea (NRF) (2017R1E1A1A01077431). Hyun Oh Song is the corresponding author.

\newpage
{\small
\bibliographystyle{ieee}
\bibliography{cvpr_2019}

\begin{thebibliography}{10}\itemsep=-1pt

\bibitem{agrawal2014}
P.~Agrawal, R.~Girshick, and J.~Malik.
\newblock Analyzing the performance of multilayer neural networks for object
  recognition.
\newblock In {\em ECCV}, 2014.

\bibitem{seanbell}
S.~Bell and K.~Bala.
\newblock Learning visual similarity for product design with convolutional
  neural networks.
\newblock In {\em SIGGRAPH}, 2015.

\bibitem{boykov_fast}
Y.~Boykov, O.~Veksler, and R.~Zabih.
\newblock Fast approximate energy minimization via graph cuts.
\newblock {\em IEEE Transactions on pattern analysis and machine intelligence},
  2001.

\bibitem{signatureVerification}
J.~Bromley, I.~Guyon, Y.~Lecun, E.~Sackinger, and R.~Shah.
\newblock Signature verification using a "siamese" time delay neural network.
\newblock In {\em NIPS}, 1994.

\bibitem{zeroshot2}
M.~Bucher, S.~Herbin, and F.~Jurie.
\newblock Improving semantic embedding consistency by metric learning for
  zero-shot classiffication.
\newblock In {\em ECCV}, 2016.

\bibitem{dch}
Y.~Cao, M.~Long, B.~Liu, and J.~Wang.
\newblock Deep cauchy hashing for hamming space retrieval.
\newblock In {\em The IEEE Conference on Computer Vision and Pattern
  Recognition (CVPR)}, June 2018.

\bibitem{cao2016deep}
Y.~Cao, M.~Long, J.~Wang, H.~Zhu, and Q.~Wen.
\newblock Deep quantization network for efficient image retrieval.
\newblock In {\em AAAI}, 2016.

\bibitem{imgnet-down}
P.~Chrabaszcz, I.~Loshchilov, and F.~Hutter.
\newblock A downsampled variant of imagenet as an alternative to the cifar
  datasets.
\newblock {\em arXiv preprint arXiv:1707.08819}, 2017.

\bibitem{contrastive}
R.~Hadsell, S.~Chopra, and Y.~Lecun.
\newblock Dimensionality reduction by learning an invariant mapping.
\newblock In {\em CVPR}, 2006.

\bibitem{jain2008fast}
P.~Jain, B.~Kulis, and K.~Grauman.
\newblock Fast image search for learned metrics.
\newblock In {\em CVPR}, 2008.

\bibitem{jeong2018}
Y.~Jeong and H.~O. Song.
\newblock Efficient end-to-end learning for quantizable representations.
\newblock In {\em ICML}, 2018.

\bibitem{adam}
D.~P. Kingma and J.~Ba.
\newblock Adam: A method for stochastic optimization.
\newblock {\em arXiv preprint arXiv:1412.6980}, 2014.

\bibitem{cifar100}
A.~Krizhevsky, V.~Nair, and G.~Hinton.
\newblock Cifar-100 (canadian institute for advanced research).
\newblock 2009.

\bibitem{li2017deep}
Q.~Li, Z.~Sun, R.~He, and T.~Tan.
\newblock Deep supervised discrete hashing.
\newblock In {\em NIPS}, 2017.

\bibitem{NIN}
M.~Lin, Q.~Chen, and S.~Yan.
\newblock Network in network.
\newblock {\em CoRR}, abs/1312.4400, 2013.

\bibitem{liu2017learning}
S.~Liu and H.~Lu.
\newblock Learning deep representations with diode loss for quantization-based
  similarity search.
\newblock In {\em IJCNN}, 2017.

\bibitem{manningbook}
C.~D. Manning, P.~Raghavan, and H.~Schutze.
\newblock {\em Introduction to Information Retrieval}.
\newblock Cambridge university press, 2008.

\bibitem{wordnet}
G.~A. Miller.
\newblock Wordnet: a lexical database for english.
\newblock {\em Communications of the ACM}, 1995.

\bibitem{hammingmetric}
M.~Norouzi, D.~J. Fleet, and R.~R. Salakhutdinov.
\newblock Hamming distance metric learning.
\newblock In {\em NIPS}, 2012.

\bibitem{ortools}
G.~OR-tools.
\newblock \url{https://developers.google.com/optimization/}, 2018.

\bibitem{imagenet}
O.~Russakovsky, J.~Deng, H.~Su, J.~Krause, S.~Satheesh, S.~Ma, Z.~Huang,
  A.~Karpathy, A.~Khosla, M.~Bernstein, A.~C. Berg, and L.~Fei-Fei.
\newblock {ImageNet Large Scale Visual Recognition Challenge}.
\newblock {\em IJCV}, 2015.

\bibitem{facenet}
F.~Schroff, D.~Kalenichenko, and J.~Philbin.
\newblock Facenet: A unified embedding for face recognition and clustering.
\newblock In {\em CVPR}, 2015.

\bibitem{domaintransduction}
O.~Sener, H.~O. Song, A.~Saxena, and S.~Savarese.
\newblock Learning transferrable representations for unsupervised domain
  adaptation.
\newblock In {\em NIPS}, 2016.

\bibitem{npairs}
K.~Sohn.
\newblock Improved deep metric learning with multi-class n-pair loss objective.
\newblock In {\em NIPS}, 2016.

\bibitem{facility}
H.~O. Song, S.~Jegelka, V.~Rathod, and K.~Murphy.
\newblock Deep metric learning via facility location.
\newblock In {\em CVPR}, 2017.

\bibitem{liftedstruct}
H.~O. Song, Y.~Xiang, S.~Jegelka, and S.~Savarese.
\newblock Deep metric learning via lifted structured feature embedding.
\newblock In {\em CVPR}, 2016.

\bibitem{deepface}
Y.~Taigman, M.~Yang, M.~Ranzato, and L.~Wolf.
\newblock Deepface: Closing the gap to human-level performance in face
  verifica- tion.
\newblock In {\em CVPR}, 2014.

\bibitem{survey_learningtohash}
J.~Wang, T.~Zhang, J.~Song, N.~Sebe, and H.~T. Shen.
\newblock A survey on learning to hash.
\newblock {\em arXiv preprint arXiv:1606.00185}, 2016.

\bibitem{triplet_video}
X.~Wang and A.~Gupta.
\newblock Unsupervised learning of visual representations using videos.
\newblock In {\em ICCV}, 2015.

\bibitem{xia2014supervised}
R.~Xia, Y.~Pan, H.~Lai, C.~Liu, and S.~Yan.
\newblock Supervised hashing for image retrieval via image representation
  learning.
\newblock In {\em AAAI}, 2014.

\bibitem{xu2015empirical}
B.~Xu, N.~Wang, T.~Chen, and M.~Li.
\newblock Empirical evaluation of rectified activations in convolutional
  network.
\newblock {\em arXiv preprint arXiv:1505.00853}, 2015.

\bibitem{zeroshot1}
Y.~Yuan, K.~Yang, and C.~Zhang.
\newblock Hard-aware deeply cascaded embedding.
\newblock In {\em ICCV}, 2017.

\bibitem{zhai2017}
A.~Zhai, D.~Kislyuk, Y.~Jing, M.~Feng, E.~Tzeng, J.~Donahue, Y.~L. Du, and
  T.~Darrell.
\newblock Visual discovery at pinterest.
\newblock In {\em Proceedings of the 26th International Conference on World
  Wide Web Companion}, 2017.

\bibitem{zhao2015deep}
F.~Zhao, Y.~Huang, L.~Wang, and T.~Tan.
\newblock Deep semantic ranking based hashing for multi-label image retrieval.
\newblock In {\em CVPR}, 2015.

\end{thebibliography}
}
\newpage
\section*{Supplementary material}
\section*{A. Proofs for equivalence}
\subsection*{A.1 Recap for flow definition}
\vspace{-2em}
\begin{align}
&(i)~ f_z(s, a_p) = k_s \nonumber\\
&(ii)~ f_z(a_p, b_{r,q}) = \bfz_p[q] \nonumber\\
&(iii)~ f_z\left( \left( b_{r,q}, b_{0,q}\right)_i \right) = \begin{cases} 1 & ~\forall i < \sum_{p:a_p\in A_r} \mathbf{z}_p[q]\\ 0 & \text{otherwise}\end{cases} \nonumber\\
&(iv)~ f_z\left( \left(b_{0,q}, t\right)_j \right) = \begin{cases} 1 & ~\forall j < \sum_{p=1}^{n_c} \mathbf{z}_p[q]\\0 & \text{otherwise}\end{cases}
\label{eqn:flow_def}
\end{align}

\subsection*{A.2 Proofs for lemma1 and lemma2}
\begin{lemma}\label{lemma:fo}
Given the minimum cost flow $\{f_o(e)\}_{e \in E'}$ of the network $G'$, the total cost of the flow is $\sum_{e\in E'} v(e)f_o(e) = \sum_{p=1}^{n_c} -{\mathbf{c}_p}^T\mathbf{z'}_p + \sum_{r=1}^l \sum_{p_1\neq p_2\in \{p|a_p\in A_r\}} \alpha {\mathbf{z'}_{p_1}}^T\mathbf{z'}_{p_2} + \sum_{p_1\neq p_2} \beta {\mathbf{z'}_{p_1}}^T\mathbf{z'}_{p_2}$.
\end{lemma}
\begin{proof}
    The total minimum cost flow is 
    \scriptsize
    \begin{align}
    \sum_{e\in E'} &v(e)f_o(e) = \underbrace{\sum_{a_p\in A} v(s,a_p) f_o(s,a_p)}_{\text{Flow from source to vertices in $A$}} + \nonumber\\
    &\underbrace{\sum_{r=0}^l \sum_{a_p\in A_r}\sum_{b_{r,q}\in B_r} v(a_p,b_{r,q}) f_o(a_p,b_{r,q})}_{\text{Flow from vertices in $A$ to vertices in $B_r$}} + \nonumber\\
    &\underbrace{\sum_{r=1}^l \sum_{b_{r,q}\in B_r}\sum_{i=0}^{g_r-1} v((b_{r,q},b_{0,q})_i) f_o((b_{r,q},b_{0,q})_i)}_{\text{Flow from vertices in $B_r$ to vertices in $B_0$}} + \nonumber\\
    &\underbrace{\sum_{b_{0,q}\in B_0}\sum_{j=0}^{n_c-1} v((b_{0,q}, t)_j) f_o((b_{0,q}, t)_j)}_{\text{Flow from vertices in $B_0$ to sink}} \nonumber
    \end{align}
    \normalsize
    Also, for $r > 0$, denote the amount of input flow at each vertex $b_{r,q} \in B_r$ given the minimum cost flow as $y'_{r,q} = \sum_{a_p\in A_r}f_o(a_p, b_{r,q}) = \sum_{p:a_p\in A_r} \mathbf{z'}_p[q]$. Also, denote the amount of input flow at each vertex $b_{0,q} \in B_0$ as $y'_{0,q} = \sum_{p:a_p\in A_0} f_o(a_p, b_{0,q}) + \sum_{r=1}^l y'_{r,q}= \sum_{p=1}^{n_c} \mathbf{z'}_p[q]$. Then, from the optimality of the minimum cost flow, $f_o((b_{r,q}, b_{0,q})_i) = \begin{cases} 1 & \forall~ i < y'_{r,q}\\ 0 & \text{otherwise}\end{cases}$ and $f_o((b_{0,q}, t)_j) = \begin{cases} 1 & ~\forall j<y'_{0,q}\\ 0 & \text{otherwise} \end{cases}$. Therefore, the total cost for optimal flow is
    \footnotesize
    \begin{align}
    &\sum_{e\in E'} v(e)f_o(e) =0+\sum_{r=0}^l \sum_{a_p\in A_r}\sum_{b_{r,q}\in B_r} -\mathbf{c}_p[q]\mathbf{z'}_p[q] + \nonumber\\
    &~~~~~~~\sum_{r=1}^l \sum_{b_{r,q}\in B_r}\sum_{i=0}^{y'_{r,q}-1} 2\alpha i+\sum_{b_{0,q}\in B_0}\sum_{j=0}^{y'_{0,q}-1} 2\beta j \nonumber\\
    &= \sum_{p} -\mathbf{c}_p^T\mathbf{z'}_p + \sum_{r=1}^l \sum_{b_{r,q}\in B_r} \alpha y'_{r,q}(y'_{r,q}-1)+\sum_{b_{0,q}\in B_0}\beta y'_{0,q}(y'_{0,q}-1) \nonumber\\
    &= \sum_{p} -\mathbf{c}_p^T\mathbf{z'}_p + \sum_{r=1}^l \sum_{b_{r,q}\in B_r} \alpha {y'_{r,q}}^2 - \sum_{r=1}^l \sum_{p:a_p \in A_r}\sum_{q=1}^d \alpha \mathbf{z'}_p[q] +\nonumber\\
    &~~~~~~~\sum_{b_{0,q}\in B_0}\beta {y'_{0,q}}^2 - \sum_{p=1}^{n_c} \sum_{q=1}^d \beta \mathbf{z}'_p[q] \nonumber\\
    &= \sum_{p} -\mathbf{c}_p^T\mathbf{z'}_p + \alpha \sum_{r=1}^l {\sum_{p:a_p \in A_r}\mathbf{z'}_p}^T \sum_{p:a_p \in A_r} \mathbf{z'}_p - \alpha \sum_{r=1}^l \sum_{p:a_p \in A_r}{\mathbf{z'}_p}^T{\mathbf{z'}_p}\nonumber\\
    &~~~~~~~+\beta{\sum_{p=1}^{n_c}\mathbf{z'}_{p}}^T\sum_{p=1}^{n_c}\mathbf{z'}_{p} - \beta \sum_{p=1}^{n_c}{\mathbf{z'}_p}^T \mathbf{z'}_p \nonumber\\
    &=\sum_{p=1}^{n_c} -{\mathbf{c}_p}^T\mathbf{z'}_p+\sum_{r=1}^l \sum_{p_1\neq p_2\in \{p|a_p \in A_r\}} \alpha {\mathbf{z'}_{p_1}}^T\mathbf{z'}_{p_2} + \sum_{\substack{p_1\neq p_2}} \beta {\mathbf{z}'_{p_1}}^T\mathbf{z'}_{p_2}.\nonumber
    \end{align}
    \normalsize
\end{proof}
\begin{lemma}\label{lemma:fz}
Given a feasible flow $\{f_z(e)\}_{e \in E'}$ of the network $G'$, the total cost of the flow is $\sum_{e\in E'} v(e)f_z(e) = \sum_{p=1}^{n_c} -{\mathbf{c}_p}^T\mathbf{z}_p + \sum_{r=1}^l \sum_{p_1\neq p_2\in \{p|a_p \in A_r\}} \alpha {\mathbf{z}_{p_1}}^T\mathbf{z}_{p_2} + \sum_{p_1\neq p_2} \beta {\mathbf{z}_{p_1}}^T\mathbf{z}_{p_2}$.
\end{lemma}
\begin{proof}
    The total cost proof is similar to \Cref{lemma:fo} except that we use the flow conditions from \Cref{eqn:flow_def} (iii) and \Cref{eqn:flow_def} (iv) instead of the optimality of the flow. 
    \scriptsize
    \begin{align}
    &\sum_{e\in E'} v(e)f_z(e) = 0+\sum_{r=0}^l \sum_{a_p\in A_r}\sum_{b_{r,q}\in B_r} -\mathbf{c}_p[q]\mathbf{z}_p[q] +\nonumber\\
    &~~~~~~~\sum_{r=1}^l \sum_{b_{r,q}\in B_r}\sum_{i=0}^{y_{r,q}-1} 2\alpha i+\sum_{b_{0,q}\in B_0}\sum_{j=0}^{y_{0,q}-1} 2\beta j \nonumber\\
    &= \sum_{p} -\mathbf{c}_p^T\mathbf{z}_p + \sum_{r=1}^l \sum_{b_{r,q}\in B_r} \alpha y_{r,q}(y_{r,q}-1)+\sum_{b_{0,q}\in B_0}\beta y_{0,q}(y_{0,q}-1) \nonumber\\
    &= \sum_{p} -\mathbf{c}_p^T\mathbf{z}_p + \sum_{r=1}^l \sum_{b_{r,q}\in B_r} \alpha {y_{r,q}}^2 - \sum_{r=1}^l \sum_{p:a_p \in A_r}\sum_{q=1}^d \alpha \mathbf{z}_p[q] +\nonumber\\
    &~~~~~~~\sum_{b_{0,q}\in B_0}\beta {y_{0,q}}^2 - \sum_{p=1}^{n_c} \sum_{q=1}^d \beta \mathbf{z}_p[q] \nonumber\\
    &= \sum_{p} -\mathbf{c}_p^T\mathbf{z}_p +\alpha \sum_{r=1}^l {\sum_{p:a_p \in A_r}\mathbf{z}_p}^T \sum_{p:a_p \in A_r} \mathbf{z}_p - \alpha \sum_{r=1}^l \sum_{p:a_p \in A_r}{\mathbf{z}_p}^T{\mathbf{z}_p}\nonumber\\
    &~~~~~~~+\beta{\sum_{p=1}^{n_c}\mathbf{z}_{p}}^T\sum_{p=1}^{n_c}\mathbf{z}_{p} - \beta \sum_{p=1}^{n_c}{\mathbf{z}_p}^T \mathbf{z}_p \nonumber\\
    &=\sum_{p=1}^{n_c} -{\mathbf{c}_p}^T\mathbf{z}_p+\sum_{r=1}^l \sum_{p_1\neq p_2\in \{p|a_p \in A_r\}} \alpha {\mathbf{z}_{p_1}}^T\mathbf{z}_{p_2} +\sum_{\substack{p_1\neq p_2}} \beta {\mathbf{z}_{p_1}}^T\mathbf{z}_{p_2} \nonumber
    \end{align}
    \normalsize
\end{proof}

\section*{B. Time complexity}
\begin{figure}[h]
    \captionsetup{font=small}
    \centering
    \begin{tikzpicture}[scale=0.9]
    \begin{axis}[
        legend pos={north west},
        xtick={16,32,64,128},
        xlabel={$d$},
        ylabel={Average wall clock run time (sec)},
        grid=major
    ]
    \addplot coordinates {
        (16, 0.00970) (32, 0.01895) (64, 0.03714) (128, 0.07361)
    };
    \addplot coordinates {
        (16, 0.01952) (32, 0.03868) (64, 0.07517) (128, 0.15410)
    };
    \addplot coordinates {
        (16, 0.03956) (32, 0.07795) (64, 0.15434) (128, 0.31331)
    };
    \addplot coordinates {
        (16, 0.08018) (32, 0.15851) (64, 0.32678) (128, 0.64261)
    };

    \legend{$n_c=64$,$n_c=128$,$n_c=256$,$n_c=512$}
    \end{axis}
    \end{tikzpicture}
    \vspace{-0.5em}
    \captionof{figure}{Average wall clock run time of computing minimum cost flow on $G'$ per mini-batch using \cite{ortools}. In practice, the run time is approximately linear in $n_c$ and $d$. Each data point is averaged over 20 runs on machines with Intel Xeon E5-2650 CPU.}
    \label{fig:mcf_runtime} 
\end{figure}
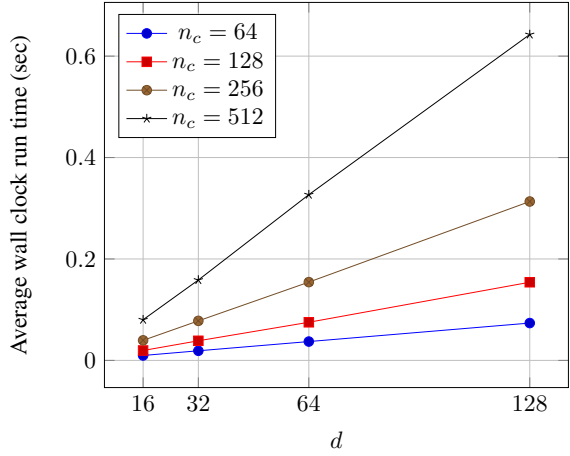

\section*{C. Effect of label remapping}
We performed the ablation study without the label remapping method. `Ours-r' in \Cref{tab:cifar_triplet} and \Cref{tab:cifar_npairs} shows the results without the remapping. The ablation study shows that the proposed hierarchical structure leads to much larger improvements than the remapping method.

\section*{D. ImagenetSplit Detail}
\Cref{tab:imgsplit_info} shows the $C_{\text{train}}$ and $C_{\text{test}}$ explicitly.

\begin{table*}[h]
\begin{adjustbox}{max width=\linewidth}
    \begin{tabular}{ll}
        \toprule
        \emph{train}& \emph{test}\\
        \cmidrule(r){1-1} \cmidrule(r){2-2}
        tench,electric ray,cock,jay&goldfish,stingray,hen,magpie\\
        common newt,spotted salamander,tree frog,loggerhead&eft,axolotl,tailed frog,leatherback turtle\\
        common iguana,agama,diamondback,trilobite&American chameleon,frilled lizard,sidewinder,centipede\\
        harvestman,quail,African grey,bee eater&scorpion,partridge,macaw,hornbill\\
        jacamar,echidna,flatworm,crayfish&toucan,platypus,nematode,hermit crab\\
        white stork,little blue heron,red-backed sandpiper,Walker hound&black stork,bittern,redshank,English foxhound\\
        Irish wolfhound,whippet,Staffordshire bullterrier,vizsla&borzoi,Italian greyhound,American Staffordshire terrier,German short-haired pointer\\
        English springer,schipperke,malinois,Siberian husky&Welsh springer spaniel,kuvasz,groenendael,malamute\\
        Cardigan,cougar,meerkat,dung beetle&Pembroke,lynx,mongoose,rhinoceros beetle\\
        ant,grasshopper,cockroach,cicada&bee,cricket,mantis,leafhopper\\
        dragonfly,Angora,hippopotamus,siamang&damselfly,wood rabbit,llama,gibbon\\
        indri,African elephant,lesser panda,sturgeon&Madagascar cat,Indian elephant,giant panda,gar\\
        hamper,bicycle-built-for-two,fireboat,cello&shopping basket,mountain bike,gondola,violin\\
        soup bowl,oxygen mask,china cabinet,Polaroid camera&mixing bowl,snorkel,medicine chest,reflex camera\\
        bottlecap,freight car,swab,fur coat&nipple,passenger car,broom,lab coat\\
        shower curtain,computer keyboard,bassoon,cliff dwelling&theater curtain,joystick,oboe,yurt\\
        loudspeaker,wool,harvester,oil filter&microphone,velvet,thresher,strainer\\
        accordion,bookcase,suit,beer glass&harmonica,wardrobe,diaper,goblet\\
        vestment,acoustic guitar,barrow,space heater&academic gown,electric guitar,shopping cart,stove\\
        crash helmet,vase,whiskey jug,cleaver&football helmet,beaker,water jug,letter opener\\
        candle,airship,combination lock,electric locomotive&spotlight,balloon,padlock,steam locomotive\\
        barometer,stethoscope,binoculars,Dutch oven&scale,syringe,projector,rotisserie\\
        frying pan,grand piano,church,swing&wok,upright,mosque,teddy\\
        knee pad,crossword puzzle,kimono,catamaran&apron,jigsaw puzzle,abaya,trimaran\\
        feather boa,cocktail shaker,holster,apiary&stole,saltshaker,scabbard,boathouse\\
        birdhouse,pirate,bobsled,slot&bell cote,wreck,dogsled,vending machine\\
        canoe,crutch,pay-phone,cinema&yawl,flagpole,dial telephone,home theater\\
        parking meter,moving van,barbell,ice lolly&stopwatch,police van,dumbbell,ice cream\\
        broccoli,zucchini,acorn squash,orange&cauliflower,spaghetti squash,butternut squash,lemon\\
        eggnog,alp,lakeside&cup,volcano,seashore\\
        \bottomrule
    \end{tabular}
    \end{adjustbox}
    \caption{Class names in train dataset and test dataset splitted from \cite{imagenet}.}
    \label{tab:imgsplit_info}
\end{table*}

\begin{table}[htbp]
\centering
\footnotesize
\begin{adjustbox}{max width=\columnwidth}
    \begin{tabular}{cc cccccccc}
        \addlinespace[-\aboverulesep]
        \cmidrule[1pt](r){1-6} \cmidrule[1pt](r){7-10}
        &      &  \multicolumn{4}{c}{\emph{test}}                                & \multicolumn{4}{c}{\emph{train}}                        \\
        \cmidrule(r){1-6} \cmidrule(r){7-10}
        &Method& SUF            & Pr@1           & Pr@4          & Pr@16         & SUF            & Pr@1           & Pr@4          & Pr@16 \\
        \cmidrule(r){1-6} \cmidrule(r){7-10}
        $k_s$
        &Metric& 1.00          &  56.78         & 55.99         & 53.95         & 1.00           &  62.64         & 61.91         & 61.22  \\
        \cmidrule(r){1-6} \cmidrule(r){7-10}
        $1$
        &LSH    & \textbf{138.83}&52.52          & 48.67         &39.71          &\textbf{135.64} & 60.45          &58.10          &54.00  \\
        &Th     & 41.21         &  54.82         & 52.88         & 48.03         & 43.19          &  61.56         & 60.24         & 58.23 \\  
        &VQ     & 22.78         &  56.74         & 55.94         & 53.77         & 40.35          & 62.54          & 61.78         & 60.98 \\ 
        &\cite{jeong2018}& 97.67&  57.63         & 57.16         & 55.76         & 97.77          &  63.85         & 63.40         & 63.39 \\
        &Ours-r & 98.12         & 58.16          & 57.39         & 56.01         & 98.45          & 64.38          &63.94          &63.86  \\
        &Ours   & 97.67         & \textbf{58.42} & \textbf{57.88}& \textbf{56.58}& 97.28          &  \textbf{64.73}& \textbf{64.63}& \textbf{64.69}\\
        \cmidrule(r){1-6} \cmidrule(r){7-10} 
        $2$
        &Th     & 14.82         &  56.55         & 55.62         & 52.90         & 15.34          &  62.41         & 61.68         & 60.89\\ 
        &VQ     & 5.63          &  56.78         & 56.00         & 53.99         & 6.94           &  62.66         & 61.92         & 61.26\\ 
        &\cite{jeong2018}   & 76.12         &  57.30         & 56.70         & 55.19         & 78.28          & 63.60      & 63.19 & 63.09\\
        &Ours-r & \textbf{101.26} & 58.18        & \textbf{57.68}& \textbf{56.43}& \textbf{100.29}& \textbf{65.00} & \textbf{64.52}& \textbf{64.42}\\
        &Ours   & 98.38&  \textbf{58.39}& 57.51& 56.09& 97.20 &  64.35& 63.91& 63.81\\
        \cmidrule(r){1-6} \cmidrule(r){7-10}
        $3$
        &Th     & 7.84          &  56.78         & 55.91         & 53.64         & 8.04           &  62.66         & 61.88         & 61.16\\
        &VQ     & 2.83          &  56.78         & 55.99         & 53.95         & 2.96           &  62.62         & 61.92         & 61.22\\ 
        &\cite{jeong2018}   & 42.12         &  56.97         & 56.25         & 54.40         & 44.36          &  62.87     & 62.22 & 61.84\\ 
        &Ours-r & \textbf{97.78} &57.58          & 57.14         & 55.70         & \textbf{97.25} & \textbf{63.95} & \textbf{63.58}&\textbf{63.48}\\
        &Ours   & 94.55& \textbf{58.19} & \textbf{57.42}& \textbf{56.02}& 93.69&63.60&63.35    &63.32\\
        \cmidrule(r){1-6} \cmidrule(r){7-10}
        $4$
        &Th     & 4.90          & 56.84          & 56.01         & 53.86         & 5.00           &  62.66         & 61.94         & 61.24 \\ 
        &VQ     & 1.91          & 56.77          & 55.99         & 53.94         & 1.97           & 62.62          & 61.91         & 61.22 \\  
        &\cite{jeong2018}& 16.19&  57.11         & 56.21         & 54.20         & 16.52          &  62.81         & 62.14         & 61.58 \\ 
        &Ours-r & \textbf{98.36} & 57.58         & 57.18         & 55.94         & \textbf{97.88} & 63.90          & 63.32         & 63.25 \\
        &Ours   & 92.18&  \textbf{58.52}& \textbf{57.79}& \textbf{56.22}& 91.27 &  \textbf{64.20}& \textbf{63.95}& \textbf{63.63}\\
        \addlinespace[-\belowrulesep]
        \cmidrule[1pt](r){1-6} \cmidrule[1pt](r){7-10}
    \end{tabular}
\end{adjustbox}
    \caption{Results with Triplet network with hard negative mining. Querying test data against a hash table built on \emph{test} set and a hash table built on \emph{train} set on Cifar-100.}
    \label{tab:cifar_triplet}
\end{table} 

\begin{table}[htbp]
\centering
\footnotesize
\begin{adjustbox}{max width=\columnwidth}
    \begin{tabular}{cc cccccccc cccccccc}
        \addlinespace[-\aboverulesep]
        \cmidrule[1pt](r){1-6} \cmidrule[1pt](r){7-10}
        &      &\multicolumn{4}{c}{\emph{test}}                                & \multicolumn{4}{c}{\emph{train}} \\
        \cmidrule(r){1-6} \cmidrule(r){7-10}
        &Method& SUF            & Pr@1           & Pr@4          & Pr@16         & SUF            & Pr@1           & Pr@4          & Pr@16 \\
        \cmidrule(r){1-6} \cmidrule(r){7-10}
        $k_s$
        &Metric &1.00            & 57.05          & 55.70         & 53.91         & 1.00           &61.78           & 60.63         & 59.73 \\
        \cmidrule(r){1-6} \cmidrule(r){7-10}
        $1$
        &LSH    & 29.74         &  53.55         &50.75          &43.03          &30.75           &59.87           &58.34          &55.35\\
        &Th     &12.72          &  54.95         & 52.60         & 47.16         & 13.65          & 60.80          & 59.49         & 57.27 \\  
        &VQ    &34.86          & 56.76          & 55.35         & 53.75         & 31.35          &  61.22         & 60.24         & 59.34 \\ 
        &\cite{jeong2018}&54.85          & 58.19 & 57.22         & 55.87         & 54.90          &  \textbf{63.11}& 62.29         & 61.94 \\
        &Ours-r & 95.30         & 58.04          & 57.31         & 56.22         &90.63           & 62.55         &62.15           &61.77\\ 
        &Ours   &\textbf{101.1}& \textbf{58.28}  & \textbf{57.79}& \textbf{56.92}& \textbf{97.47} &  63.06         & \textbf{62.62}& \textbf{62.44} \\
        \cmidrule(r){1-6} \cmidrule(r){7-10}
        $2$
        &Th     &5.09           &  56.52         & 55.28         & 53.04         & 5.36           &  61.65         & 60.50         & 59.50  \\ 
        &VQ     &6.08           &  57.13         & 55.74         & 53.90         & 5.44           &  61.82         & 60.56         & 59.70  \\ 
        &\cite{jeong2018}&16.20          &  57.27         & 55.98         & 54.42         & 16.51          &  61.98         & 60.93         & 60.15 \\
        &Ours-r & 66.74 &  57.73 & 57.01 & 55.66 & 67.46 & 62.76 & 61.87 & 61.36 \\
        &Ours   &\textbf{69.48} &  \textbf{57.60}& \textbf{56.98}& \textbf{55.82}& \textbf{69.91} &  \textbf{62.19}& \textbf{61.71}& \textbf{61.27} \\
        \cmidrule(r){1-6} \cmidrule(r){7-10}
        $3$
        &Th     & 3.10           &  56.97         & 55.56         & 53.76         & 3.21           &  61.75         & 60.66         & 59.73 \\
        &VQ     & 2.66           &  57.01         & 55.69         & 53.90         & 2.36           &  61.78         & 60.62         & 59.73 \\ 
        &\cite{jeong2018}& 7.25           &  57.15         & 55.81         & 54.10         & 7.32           &  61.90         & 60.80         & 59.96 \\ 
        &Ours-r & 55.83 & \textbf{57.81} &56.55 &55.11 &57.11& 62.20& \textbf{61.50} & 60.90 \\
        &Ours   & \textbf{57.09}&  57.56& \textbf{56.70}& \textbf{55.41}& \textbf{58.62} &  \textbf{62.30}& 61.44& \textbf{60.91} \\
        \cmidrule(r){1-6} \cmidrule(r){7-10}
        $4$
        &Th     & 2.25           &  57.02         & 55.64         & 53.88         & 2.30           & 61.78          & 60.66         & 59.75 \\ 
        &VQ     & 1.66           & 57.03          & 55.70         & 53.91         & 1.55           &  61.78         & 60.62         & 59.73 \\  
        &\cite{jeong2018}   & 4.51           & 57.15          & 55.77         & 54.01         & 4.52           & 61.81          & 60.69         & 59.77 \\ 
        &Ours-r & 48.04 & \textbf{57.76} & 56.70 & 55.11 & 49.73 & 62.12 & 61.30 & 60.74\\
        &Ours   & \textbf{49.43} & 57.75 & \textbf{56.79}& \textbf{55.50}& \textbf{50.80} & \textbf{62.43} & \textbf{61.65}& \textbf{61.01} \\
        \addlinespace[-\belowrulesep]
        \cmidrule[1pt](r){1-6} \cmidrule[1pt](r){7-10}
    \end{tabular}
\end{adjustbox}
    \caption{Results with Npairs \cite{npairs} network. Querying test data against a hash table built on \emph{test} set and a hash table built on \emph{train} set on Cifar-100.}
    \label{tab:cifar_npairs}
\end{table}

\end{document}